\documentclass{article}

\usepackage[main, final]{neurips_2026}
\usepackage{hyperref}
\usepackage{xcolor}
\setcitestyle{open={(},close={)}}
\hypersetup{hidelinks}

\usepackage[utf8]{inputenc} 
\usepackage[T1]{fontenc}    
\usepackage{hyperref}       
\usepackage{url}            
\usepackage{booktabs}       
\usepackage{amsfonts}       
\usepackage{nicefrac}       
\usepackage{microtype}      
\usepackage{xcolor}         
\usepackage{url}
\usepackage{booktabs}
\usepackage{amsfonts}
\usepackage{nicefrac}
\usepackage{microtype}
\usepackage{xcolor}
\usepackage{graphicx}
\usepackage{amsmath,amssymb}
\usepackage{amsthm}
\usepackage{multirow}
\usepackage{subcaption}
\usepackage{array}
\usepackage{tabularx}
\usepackage{enumitem}
\usepackage{float}
\usepackage{placeins}
\usepackage{cleveref}
\usepackage[most]{tcolorbox}
\usetikzlibrary{arrows.meta,calc,positioning}

\graphicspath{{figures_NL/}{figures_python/}{figures_html/}{figures_latex/}{figures/}}
\hypersetup{
    hidelinks,
    colorlinks=true,
    allcolors=blue!60!black
}
\newcommand{\tok}[1]{\texttt{#1}}

\newif\ifshowcomments
\showcommentstrue
\ifshowcomments
  \newcommand{\draftcomment}[3]{{\textcolor{#3}{[#2: \emph{#1}]}}}
\else
  \newcommand{\draftcomment}[3]{}
\fi

\newtcolorbox{observationbox}[1][]{
  breakable,
  enhanced,
  colback=blue!2,
  colframe=blue!55!black,
  boxrule=0.6pt,
  arc=2pt,
  left=6pt,
  right=6pt,
  top=5pt,
  bottom=5pt,
  title=Observation,
  fonttitle=\bfseries,
  #1
}

\newtcolorbox{takeawaybox}[1][]{
  breakable,
  enhanced,
  colback=green!3,
  colframe=green!45!black,
  boxrule=0.6pt,
  arc=2pt,
  left=6pt,
  right=6pt,
  top=5pt,
  bottom=5pt,
  title=Takeaway,
  fonttitle=\bfseries,
  #1
}

\usepackage{thm-restate}

\newcommand\upd[1]{{#1}}

\usepackage{graphicx}
\usepackage{enumitem}
\usepackage{tikz}
\usetikzlibrary{
  arrows.meta,
  positioning,
  calc,
  matrix,
  fit,
  backgrounds,
  decorations.pathreplacing
}

\definecolor{YLHybrid}{HTML}{D95AA1}
\definecolor{YLHybridLight}{HTML}{FCE6F2}
\definecolor{YLAttn}{HTML}{287A80}
\definecolor{YLAttnLight}{HTML}{E4F4F5}
\definecolor{YLSoft}{HTML}{F4F5F7}
\definecolor{YLDark}{HTML}{3F4A56}

\definecolor{pm}{HTML}{1AA3D4}
\definecolor{et}{HTML}{A67868}
\definecolor{sc}{HTML}{F57F35}

\usepackage{tabularx}
\usepackage{array}
\usepackage{makecell}
\usepackage[table]{xcolor}
\usepackage{arydshln}

\definecolor{pm}{HTML}{1AA3D4} 
\definecolor{sc}{HTML}{F57F35} 

\title{Comparing Transformers and Hybrid Models at the Token Level}

\author{%
  Yanhong Li \\
  Allen Institute for AI \\
  \texttt{yanhongl@allenai.org} \\
  \And
  William Merrill \\
  Allen Institute for AI \\
  \texttt{willm@allenai.org}
}

\begin{document}

\maketitle

\begin{abstract}
    Hybrid language models combine attention and recurrence: recurrent layers theoretically ameliorate pure transformers' state-tracking limitations, and hybrids can empirically outperform pure transformers in loss and downstream evaluations \citep{waleffe2024empirical,merrill2026olmohybrid}. Which data or capabilities drive these gains, and how closely they reflect hybrids' theoretical advantages, remains unclear. We address this \upd{using the open-weight Olmo 3 \citep{olmo2025olmo3} and Olmo Hybrid \citep{merrill2026olmohybrid} checkpoints, comparing losses for the matched transformer and hybrid} at identical target tokens and prefixes, stratified by natural token tags, copy features, delimiter structure, and controlled synthetic probes. The hybrid lowers loss on most tag families, especially open-class content words, with smaller gains on many closed-class function words. Across prose, code, and markup, its advantage is larger on opening than corresponding closing delimiters and nearly vanishes on repeated $n$-grams. Probes likewise favor the hybrid on pronoun memory and entity tracking, but the transformer on bracket matching requiring closing delimiters. These patterns suggest hybrids' recurrent layers improve predictions using document semantic state, while attention helps predict tokens predictable by $n$-gram copying or syntactic bracket matching. Proof-of-concept filtered evaluations show how token-level decompositions can sharpen pretraining diagnostics for hybrid architectures.
\end{abstract}

\section{Introduction}

Hybrid language models that mix attention with recurrent sequence layers have recently challenged the attention-only transformer as the default architecture for large-scale language modeling. Their theoretical motivation is complementary expressivity: attention retrieves past tokens for copying and syntactic tasks like bracket matching, while recurrence helps construct and update latent state in ordered state-tracking computations \citep{merrill2024illusion,grazzi2025unlocking}.
Hybrids theoretically inherit both strengths and can empirically outperform transformers with the same pretraining budget in loss and downstream benchmarks \citep{waleffe2024empirical,merrill2026olmohybrid}. Yet it remains unclear which finegrained predictions hybrid layers improve, or how these improvements relate to their theoretical advantages. We study these differences \emph{at the token level}, asking:

\begin{enumerate}[nosep]
    \item \textbf{\emph{Which tokens and prediction contexts do hybrids predict better than transformers?}} Which individual token occurrences account for their lower average loss relative to a matched transformer? \item \textbf{\emph{Do these occurrences reflect the component architectures' theoretical advantages?}} Do hybrid-favored tokens reflect recurrence's expressivity advantages, such as ordered state tracking, and do transformer-favored tokens reflect attention's, such as visible-prefix recall?
\end{enumerate}

We compare two released 7B-scale models from the same recipe family, Olmo~3~7B and Olmo Hybrid~7B \citep{olmo2025olmo3,merrill2026olmohybrid}. At position $i$, we compute paired losses for observed token $x_i$ under prefix $x_{<i}$:
\[
\ell_i^{\mathrm{Tr}}=-\log p_{\mathrm{Tr}}(x_i\mid x_{<i}),\qquad
\ell_i^{\mathrm{Hyb}}=-\log p_{\mathrm{Hyb}}(x_i\mid x_{<i}),
\]
with gap $\Delta_i=\ell_i^{\mathrm{Tr}}-\ell_i^{\mathrm{Hyb}}$. Positive $\Delta_i$ means higher hybrid probability (lower NLL). This moves from average loss to asking which token-level prediction events produce the gain.

Our observational analysis computes $\Delta_i$ over prose, code, and markup, aligning targets to POS tags for prose and source-level categories for structured text, such as identifiers, strings, comments, text nodes, attributes, commands, brackets, and tags. Raw tag means describe absolute hybrid gains on actual corpus positions; complementary regression checks test persistence under controls for difficulty, frequency, position, subword status, and local reuse. 

The natural-token results are non-uniform: the hybrid has lower loss on most token families, with the largest advantage on open-class content-bearing categories. In prose, content words have a larger raw gap than function words, and the aggregate contrast remains after controls. In structured domains, hybrid-favored categories include identifiers, strings, comments, text nodes, attribute values, and commands. The advantage is smaller on closing delimiters and rigid formatting tokens, and approaches zero on long repeated $n$-grams. Thus the hybrid wins overall, especially on open-class words, while the transformer remains competitive, sometimes favored, in contexts suited to attention: retrieving visible-prefix material or satisfying an already-open structural obligation.

Our second analysis uses controlled synthetic probes to separate delayed information from the computation required at the target. Pronoun memory introduces people with roles and later queries the pronoun for a role's occupant. Entity tracking binds entities to attributes and later queries which entity has an attribute. Structural closure opens a bracket or tag region and later queries the required closing token. The first two favor the hybrid and require readout of a maintained role or entity--attribute binding. Closure also involves delay, but its answer is determined by a visible-prefix opener, and it favors the transformer. Thus the distinction is not simply dependency length, but whether the target is primarily a state-conditioned choice, visible-prefix copy, or closure decision.

We consider theoretical explanations for the token-level patterns we find and hypothesize that discourse state tracking abilities may play a central role. Controlled state-tracking benchmarks often ask models to update a small, fixed set of symbols following a sequence of instructions and report the final symbol or value. These closed-world tasks isolate ordered update and readout, but become less informative for distinguishing future recurrent layers once architectures solve the fixed-slot setting. This interpretation motivates exploring an open-world extension to state tracking. Inputs can introduce people, objects, variables, or document regions over time, change their attributes or relations, and require later context-dependent predictions of entities, values, or content words. Such an extension would retain state tracking while replacing fixed symbolic slots with the growing, relational state of natural text, code, and markup, potentially providing a more demanding testbed for linear-RNN and hybrid sequence layers.

\upd{In addition, our findings suggest filtered token losses are a promising way to make fine-grained comparisons between}
architectures during pretraining.
As a proof of concept, we compare 1B-scale development runs comparing a Transformer, a Hybrid, and a Pure RNN.
Aggregate validation loss compresses distinct regimes into one number. On the other hand, filtered token losses \upd{reveal fine-grained differences}: hybrid-favored non-copy tokens separate architectures more sharply, while a copy-only filter reveals where attention-based models outperform the Pure RNN on visible-prefix reuse. Computed from the same per-token NLL as standard validation, these filters add little overhead and resolve training progress by capability. Reporting them alongside aggregate validation loss can show whether a design improves overall perplexity and which predictive capabilities it improves or sacrifices.

\section{Expressivity Background and Empirical Hypotheses}
\label{sec:architecture-state}

Expressivity theory separates next-token computations by architectural primitive: attention, recurrence, and their combination in hybrid models.

\textbf{Transformer expressivity: copy/recall and structural matching.}
Attention directly selects positions in the visible prefix. Under standard formalizations, this supports recall-style computations such as $n$-gram retrieval and many bracket- or structural-matching problems by exposing the relevant earlier position or opener \citep{weiss2021thinkingliketransformers,yao2023selfattentionnetworksprocessbounded}. Thus attention should be especially useful when predicting tokens by reusing visible material or matching an explicit structural opener:
\begin{enumerate}[label=(\arabic*),ref=\arabic*,leftmargin=2.2em]
    \item \label{ex:copying}
    \textit{John works for the National Hamburger Association of America. The National Hamburger Association of \underline{America} $\ldots$}

    \item \label{ex:parens}
    \texttt{( [ ] [ \{ ( ) \} ] \underline{)}}.
\end{enumerate}
The examples require visible-prefix copy/recall and matching a visible structural opener, respectively. Pure recurrent models with bounded state have known copy/recall limitations: copying or recalling arbitrary prefix information can require more information than a fixed-size state can retain \citep{arora2024zoology,jelassi2024repeat,merrill2026olmohybrid}. \upd{\footnote{Additional work on copying and retrieval is discussed in Appendix~\ref{app:copying_related_work}.}}

\textbf{RNN expressivity: ordered state tracking.}
The complementary limitation of fixed-depth transformers is ordered state composition. Under standard fixed-depth and log-precision assumptions, transformer next-token predictors are contained in the low-depth threshold-circuit class $\mathsf{TC}^0$.
Consequently, they cannot express general $\mathsf{NC}^1$-complete ordered state-tracking problems unless $\mathsf{TC}^0=\mathsf{NC}^1$ \citep{merrill-sabharwal-2023-parallelism,chiang2025transformers}. In contrast, modern linear RNNs with sufficiently expressive transition matrices, including \upd{DeltaNet and Gated DeltaNet (GDN) variants} with negative eigenvalues, can represent such state-tracking computations \upd{\citep{merrill2024illusion,grazzi2025unlocking,merrill2026olmohybrid,yang2025gated}}. These results motivate the complementary hypothesis that predictions requiring ordered state updates should favor models with recurrent layers over pure transformers. A simple program-state example is:
\begin{enumerate}[label=(\arabic*),ref=\arabic*,leftmargin=2.2em,start=3]
    \item \label{ex:program-state}
    \texttt{a,b,c = 1,2,3; a,c = c,a; assert a == }\underline{3}.
\end{enumerate}

\textbf{Hybrid expressivity: composing recall with state.}
Hybrid models combine attention-based visible-prefix recall with recurrent ordered state updates. They should thus handle contexts like \Cref{ex:copying,ex:parens,ex:program-state}, where pure RNNs or transformers lack one primitive. Beyond combining independent capabilities, \citet{merrill2026olmohybrid} show that a GDN--attention hybrid can solve \emph{state-based recall}: tracking pointer updates, then using the resulting pointer to retrieve a prefix value. Neither a pure transformer nor a pure GDN can express the full problem under standard assumptions. A code-like example is:
\begin{enumerate}[label=(\arabic*),ref=\arabic*,leftmargin=2.2em,start=4]
    \item \label{ex:state-based-recall}
    \texttt{bits = [0,1,0,0,$\ldots$]; a,b,c = 3,1,2; a,c = c,a; assert bits[a] == }\underline{0}.
\end{enumerate}

This expressivity picture motivates our empirical hypotheses. We treat each prefix as inducing a latent discourse/program state and each target token as querying or updating that state. Some targets are recoverable from the visible prefix; others require ordered state construction and state-conditioned readout; others compose both. We therefore expect attention-heavy transformers to be competitive on visible-prefix reuse and structural closure, recurrent layers to help state-conditioned predictions, and hybrids to be strongest when both are required. \upd{Going beyond these synthetic cases, we test where differences emerge between architectures when predicting natural-language, code, and markup tokens.}

\section{Empirical Methodology}
\label{sec:empirical_method}

We identify which \emph{individual next-token predictions} drive the matched hybrid--transformer loss gap, using target positions as our units. For each position $i$ in a packed sequence $x_{1:L}$, both models score the same observed target $x_i$ under prefix $x_{<i}$. Let
\[
\ell_i^{\mathrm{Tr}}
    = -\log p_{\mathrm{Tr}}(x_i \mid x_{<i}),
    \qquad
\ell_i^{\mathrm{Hyb}}
    = -\log p_{\mathrm{Hyb}}(x_i \mid x_{<i}),
\]
where $\mathrm{Tr}$ denotes the transformer and $\mathrm{Hyb}$ denotes the hybrid model.  We define the paired token-level loss gap\footnote{
Because $\Delta_i$ is a log-probability difference, a mean gap of
$\bar{\Delta}$ nats corresponds to a geometric-mean probability ratio of
$\exp(\bar{\Delta})$ in favor of the hybrid.  For example, $0.04$ nats is about a
$4.1\%$ probability ratio.  It is not a log-odds ratio.
}
\[
\Delta_i
    = \ell_i^{\mathrm{Tr}} - \ell_i^{\mathrm{Hyb}}
    = \log p_{\mathrm{Hyb}}(x_i \mid x_{<i})
      - \log p_{\mathrm{Tr}}(x_i \mid x_{<i}).
\]
Thus $\Delta_i>0$ means that, at position $i$, the hybrid assigns higher probability to the observed next token than the transformer.  All natural-token analyses in this section aggregate these per-position quantities.

\textbf{Models, domains and evaluation.}
We compare released 7B-scale models from the same recipe family: \textbf{Olmo~3~7B} and \textbf{Olmo Hybrid~7B} \citep{olmo2025olmo3,merrill2026olmohybrid}.\footnote{\upd{We also evaluate NVIDIA's Transformer--Mamba-2 Hybrid pair and pure Mamba-2 model (Appendix~\ref{app:nvidia_replication}), but do not focus on this comparison: besides the sequence mixer, the Transformer and Hybrid differ in normalization, positional encoding, attention configuration, and activation functions.}} Their closely matched tokenizer, data mixture, and training recipe make the per-token gap primarily reflect the sequence mixer. We evaluate prose (PG-19, News, Wikipedia, essay, textbooks and scientific papers) and structured code and markup (Python, HTML, \LaTeX{}), packed into length-$L{=}8192$ sequences, computing next-token NLL at every position. Unless stated otherwise, all 7B token-level and synthetic analyses use the final released checkpoint pair. Appendix~\ref{app:method_details} gives further preprocessing and checkpoint details.

\textbf{Token tagging and alignment.}
After computing $\Delta_i$ at every LM target position, we assign surface tags in two steps: tag spans in the original source, then align them to decoded LM tokens by character-span overlap. Domains share this alignment procedure but differ in taggers and inventories. Prose words receive Brown POS tags; Python, HTML, and \LaTeX{} spans come from lightweight tokenizers or parsers, with categories such as identifiers, strings, comments, delimiters, tags, attributes, commands, and text nodes. Multi-tag attribution lets the same $\Delta_i$ contribute to every overlapping tag's summary. Full taxonomies and alignment rules appear in Appendix~\ref{app:tagging_details} and Appendices~\ref{sec:python_analysis}--\ref{sec:html_analysis}. We also record each LM token's word-position type: whole word, prefix, middle subword, or suffix.

\textbf{Analysis I: Tag-stratified raw summaries.}
Raw means unadjusted: we average the paired loss gap over observed corpus positions assigned to a tag, without reweighting or controls. This describes how much lower the hybrid's NLL is than the transformer's on those positions. For a tag or position set $\tau$, let $\mathcal{I}_{\tau}$ contain its assigned target positions. We report
\[
\widehat{\Delta}(\tau)
    =
    \frac{1}{|\mathcal{I}_{\tau}|}
    \sum_{i\in \mathcal{I}_{\tau}} \Delta_i
    \;=\;
    \mathbb{E}[\Delta_i \mid i\in \mathcal{I}_{\tau}].
\]
Examples of $\tau$ include \textsc{Noun}, \textsc{Open Bracket}, and \textsc{HTML open tag}. Multi-tag tokens belong to each overlapping tag's $\mathcal{I}_{\tau}$, contributing their gap to every corresponding raw summary.

\textbf{Analysis II: Regression robustness check.}
Raw tag means provide the main descriptive summary of the hybrid--transformer gap on actual positions with each tag, but may be confounded by difficulty, frequency, position, subword status, and local reuse. These tag correlates can independently affect $\Delta_i$; for example, content and function categories differ substantially in token frequency. We therefore control for them with linear regressions:

\begin{equation}
\label{eq:prose_reg}
\vspace{-1pt}
\begin{array}{@{}l@{\hspace{1em}}l@{}}
\begin{aligned}
\Delta_i \;\sim\;& \textsc{domain}_i + \textsc{tag}_i + \textsc{wpos}_i \\ &+ \textsc{relpos}_i + \bar\ell_i + \bar\ell_i^{2} \\ &+ \sum_{k\in\{1,2,3,4\}} \textsc{copy}_{k,i} \\
&+ \log \textsc{prevdist}_i + \log \textsc{freq}(x_i).
\end{aligned}
&
\text{\scriptsize\setlength{\tabcolsep}{1pt}
\renewcommand{\arraystretch}{0.8}
\begin{tabular}{@{}l p{26em}@{}}
$\Delta_i$        & paired NLL gap $\ell^{\mathrm{Tr}}_i-\ell^{\mathrm{Hyb}}_i$ ($>0$ = hybrid better) \\
\textsc{domain}   & prose source fixed effects \\
\textsc{tag}      & either coarse POS family or aggregate word class, depending on the model \\
\textsc{wpos}     & subword position in word (whole/prefix/middle/suffix) \\
\textsc{relpos}   & relative position in packed sequence \\
$\bar\ell_i$      & mean NLL; difficulty proxy ($\bar\ell_i^{2}$ allows curvature) \\
\textsc{copy}$_k$  & prefix-reuse features for $k=1,2,3,4$; $\textsc{copy}_{1,i}$ is same-token reuse (the target token type appeared earlier in the prefix), and $\textsc{copy}_{k,i}$ for $k\ge2$ indicates that the target completes a repeated $k$-gram \\
\textsc{prevdist} & distance to previous occurrence of the same token type (log-scaled) \\
\textsc{freq}     & empirical target-token type frequency in the pooled regression sample \\
\end{tabular}}
\vspace{-1pt}
\end{array}
\end{equation}

We fit Equation~\ref{eq:prose_reg} separately with two tag inventories and the same controls. The \textbf{coarse-tag model} uses the full coarse Brown-POS family, giving detailed adjusted effects for nouns, verbs, auxiliaries, brackets, punctuation, and other categories. The \textbf{aggregate word-class model} replaces \textsc{tag} with \textsc{Content}/\textsc{Function}/\textsc{Other},\footnote{\textsc{Content} includes open-class lexical categories such as nouns, main verbs, adjectives, and adverbs. \textsc{Function} includes closed-class grammatical categories such as determiners, prepositions, conjunctions, pronouns, auxiliaries, modals, wh-categories, particles, and infinitival \textsc{TO}. \textsc{Other} includes punctuation, brackets, symbols, numerals, and remaining tags.} giving the aggregate rows in Figure~\ref{fig:tag_raw_vs_effect}.

Raw gaps show a token family's absolute hybrid advantage or disadvantage; regression effects show whether it is more or less hybrid-favored than expected after controls. These summaries answer complementary questions.

\begin{figure}[H]
\centering
\begin{tikzpicture}[font=\sffamily]
\definecolor{pm}{HTML}{1AA3D4}
\definecolor{et}{HTML}{A67868}
\definecolor{sc}{HTML}{F57F35}
\definecolor{boxgray}{HTML}{D8DDE5}
\newlength{\cardW}       \setlength{\cardW}{0.485\linewidth}
\newlength{\cardWfull}   \setlength{\cardWfull}{0.98\linewidth}
\newlength{\cardHtop}    \setlength{\cardHtop}{1.7cm}
\newlength{\cardHbot}    \setlength{\cardHbot}{1.2cm}
\newlength{\rowGap}      \setlength{\rowGap}{0.08cm}
\newlength{\colGap}      \setlength{\colGap}{0.16cm}
\newlength{\barW}        \setlength{\barW}{0.1cm}
\newlength{\padX}        \setlength{\padX}{0.12cm}
\newlength{\labelWtop}   \setlength{\labelWtop}{1.6cm}
\newlength{\labelWbot}   \setlength{\labelWbot}{2.1cm}
\newlength{\ansWtop}     \setlength{\ansWtop}{1.3cm}
\newlength{\ansWbot}     \setlength{\ansWbot}{1.55cm}
\newlength{\ansH}        \setlength{\ansH}{0.45cm}
\newlength{\promptOffTop}\setlength{\promptOffTop}{1.85cm}
\newlength{\promptOffBot}\setlength{\promptOffBot}{2.4cm}
\tikzset{
  card/.style={
    draw=boxgray, rounded corners=1.5mm, line width=0.5pt,
    inner sep=0pt, outer sep=0pt, fill=white
  },
  ansbox/.style={
    rounded corners=1mm, line width=0.7pt, minimum height=\ansH,
    inner sep=1pt, align=center, font=\bfseries\footnotesize
  }
}
\node[card, anchor=north west, minimum width=\cardW, minimum height=\cardHtop] (c1) at (0,0) {};
\node[card, anchor=north west, minimum width=\cardW, minimum height=\cardHtop]
  at ([xshift=\colGap]c1.north east) (c3) {};
\node[card, anchor=north west, minimum width=\cardWfull, minimum height=\cardHbot]
  at ([yshift=-\rowGap]c1.south west) (c2) {};
\fill[pm] (c1.north west) rectangle ([xshift=\barW]c1.south west);
\fill[sc] (c3.north west) rectangle ([xshift=\barW]c3.south west);
\fill[et] (c2.north west) rectangle ([xshift=\barW]c2.south west);
\node[anchor=west, text=pm, font=\bfseries\small, align=left, text width=\labelWtop]
  at ([xshift=\barW+\padX]c1.west) {Pronoun\\memory};
\node[anchor=west, text=sc, font=\bfseries\small, align=left, text width=\labelWtop]
  at ([xshift=\barW+\padX]c3.west) {Structural\\closure};
\node[anchor=west, text=et, font=\bfseries\small, align=left, text width=\labelWbot]
  at ([xshift=\barW+\padX]c2.west) {Entity\\tracking};
\node[anchor=north west, font=\ttfamily\scriptsize, align=left,
      text width=\dimexpr\cardW-\promptOffTop-\ansWtop-0.3cm\relax]
  at ([xshift=\promptOffTop, yshift=-0.1cm]c1.north west)
  {Liam is the violinist.\\ Naomi is the pilot.\\
   \dots\ $d$ filler tokens \dots\\ the violinist reviewed the report, and};
\node[anchor=north west, font=\ttfamily\scriptsize, align=left,
      text width=\dimexpr\cardW-\promptOffTop-\ansWtop-0.3cm\relax]
  at ([xshift=\promptOffTop, yshift=-0.1cm]c3.north west)
  {\texttt{<header>}\\ \qquad counter += 1\\ \qquad \dots\ $d$ filler tokens \dots};
\node[anchor=north west, font=\ttfamily\scriptsize, align=left, text width=8.5cm]
  at ([xshift=\promptOffBot, yshift=-0.1cm]c2.north west)
  {Julia carried the orange notebook. Sofia carried the green folder.\\
   \dots\ $d$ filler tokens \dots\ Q: Who carried the green folder?\\
   (A) Sofia \quad (B) Julia \quad Answer:};
\node[ansbox, draw=pm, text=pm, anchor=north east, minimum width=\ansWtop]
  at ([xshift=-0.12cm, yshift=-0.12cm]c1.north east) {he};
\node[ansbox, draw=sc, text=sc, anchor=north east, minimum width=\ansWtop]
  at ([xshift=-0.12cm, yshift=-0.12cm]c3.north east) {\texttt{</header>}};
\node[ansbox, draw=et, text=et, anchor=north east, minimum width=\ansWbot]
  at ([xshift=-0.12cm, yshift=-0.12cm]c2.north east) {Sofia};
\node[anchor=north east, text=black!60, font=\itshape\scriptsize]
  at ([xshift=-0.12cm, yshift=-0.65cm]c1.north east) {vs.\ she};
\node[anchor=north east, text=black!60, font=\itshape\scriptsize]
  at ([xshift=-0.12cm, yshift=-0.65cm]c3.north east) {NLL on closer};
\node[anchor=north east, text=black!60, font=\itshape\scriptsize]
  at ([xshift=-0.12cm, yshift=-0.65cm]c2.north east) {vs.\ Julia};
\end{tikzpicture}
\caption{\textbf{Synthetic probe examples.} We vary the distance $d$ between an antecedent/opener and a scored target token.}
\vspace{-5pt}
\label{fig:synthetic_examples}
\end{figure}

\textbf{Analysis III: Controlled synthetic probes.}
Observational natural-token analyses locate the gap in real text, but surface tags do not directly specify the scored decision. We therefore add three controlled synthetic probe families (Figure~\ref{fig:synthetic_examples}):
\begin{itemize}[nosep,leftmargin=1.4em]
    \item \textsc{Pronoun memory}. Introduce two people with different genders and roles, insert filler, then query a role and score its matching pronoun, e.g., \emph{he} versus \emph{she}. \item \textsc{Entity tracking}. Bind two entities to different attributes, insert filler, then ask which entity has a queried attribute. \item \textsc{Structural closure}. Open a bracketed or tagged region, insert filler, then score the required closer. \end{itemize} For each family, we fix the surface template and filler distribution, varying the cue-to-target distance $d\in\{32,64,128,256,512,1024\}$.

For the first two probes, we score them contrastively using accuracy $\mathbf{1}[m>0]$ and margin
\[
m = \log p(y^+ \mid c) - \log p(y^- \mid c),
\]
where $y^+$ is the correct token, $y^-$ is a matched distractor, and $c$ is the shared prefix.  \textsc{Structural closure} instead tests visible structural closure, so we report NLL on the closing token,
\[
\ell_{\mathrm{close}}=-\log p(y_{\mathrm{close}}\mid c).
\]

\section{Empirical Results}
\label{sec:empirical_results}

Section~\ref{sec:natural_token_results} locates the hybrid--transformer gap in natural prose, code, and markup, distinguishing absolute raw gaps from controlled regression effects. Section~\ref{sec:synthetic_results} tests whether the same split appears in controlled synthetic probes that manipulate target type and antecedent distance.

\subsection{Natural-token analysis}
\label{sec:natural_token_results}

\begin{figure}[t]
\centering
\includegraphics[width=\linewidth]{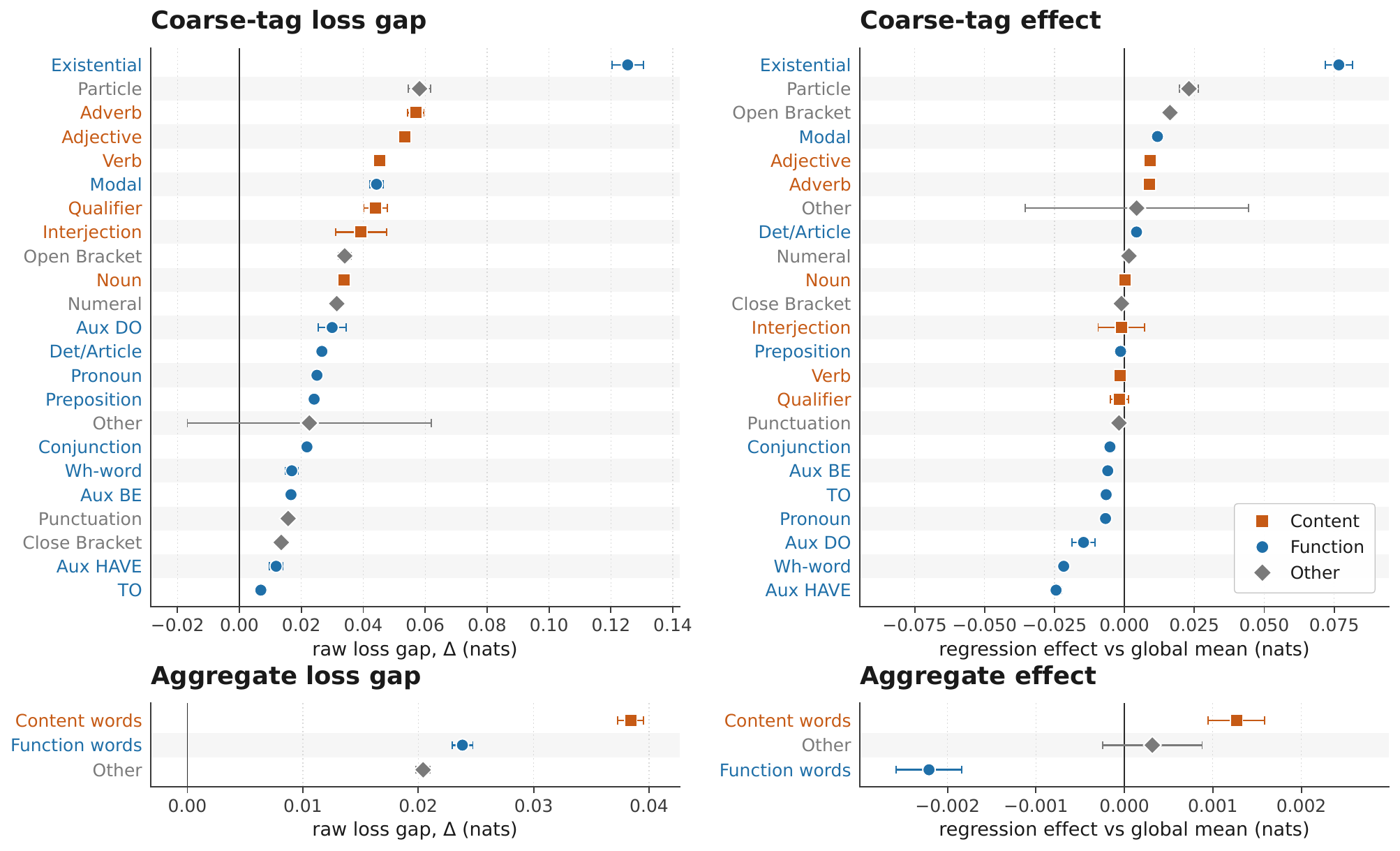}
\caption{\textbf{Raw and adjusted tag effects in prose.}
Left panels show raw paired loss gaps, where positive values mean the hybrid has lower NLL on that token family.
Right panels show regression-adjusted effects relative to the global mean from the regressions.
Top panels use the full coarse POS taxonomy; bottom panels use the separately fitted three-way aggregate model (content/function/other).}
\label{fig:tag_raw_vs_effect}
\end{figure}

\begin{figure}[t]
\centering
\includegraphics[width=\linewidth]{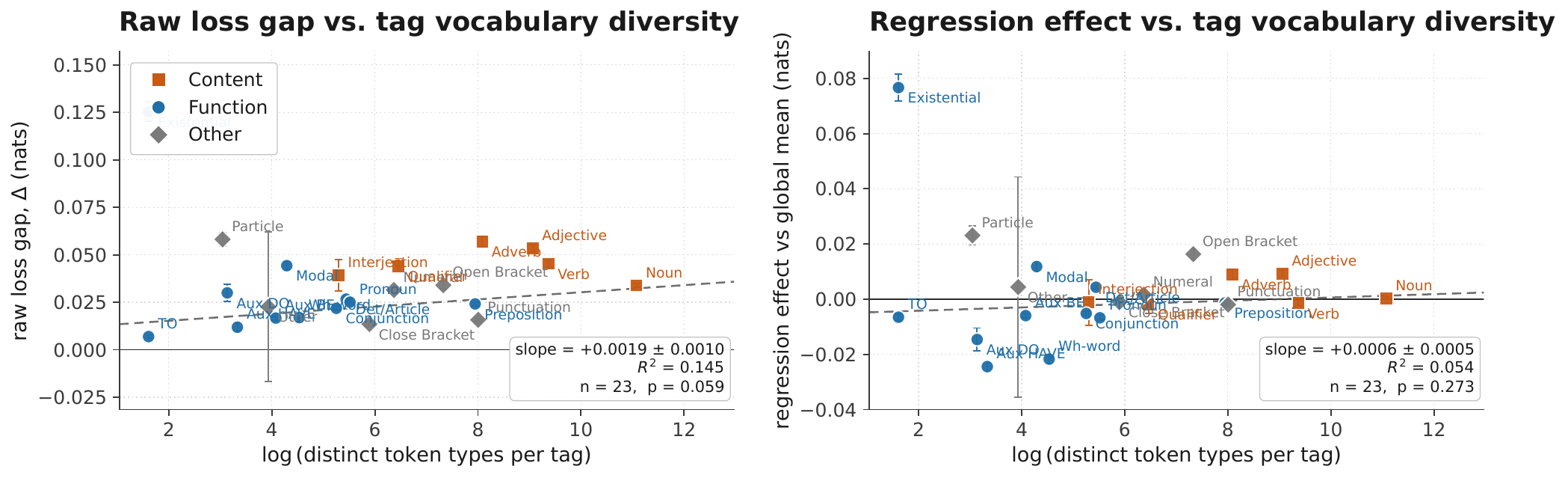}
\caption{\textbf{Tag vocabulary size is associated with the prose tag effects.}
Each point is a coarse prose tag from Figure~\ref{fig:tag_raw_vs_effect}, marked by aggregate word class. The x-axis is $\log|\mathcal{V}_\tau|$, where $\mathcal{V}_\tau$ is the set of distinct target-token types observed with tag $\tau$. Left: raw paired loss gap $\bar\Delta(\tau)$. Right: regression-adjusted tag effect relative to the global mean.}
\vspace{-10pt}
\label{fig:tag_vocab_scatter}
\end{figure}

\textbf{Hybrid advantage across token tags is largest for open-class content words.} \Cref{fig:tag_raw_vs_effect} shows a non-uniform hybrid advantage across most prose tag families. In the aggregate raw panel, content words have the largest advantage ($0.0384$ nats), versus $0.0238$ nats for function words: a difference of $0.0146$ nats (about $61\%$ larger). This ordering survives controls for difficulty, token/subword position, sequence position, local reuse, previous-token distance, and token frequency: content remains above the global mean and function below it. This stricter regression test also leaves many open-class categories more hybrid-favored than closed-class or highly constrained categories. Individual function tags are more heterogeneous, but several (e.g., auxiliaries, wh-words, \textsc{TO}, pronouns) remain among the least hybrid-favored.

Function words tend to have small, closed inventories, whereas content words tend to be open class. We therefore test whether tag effects vary continuously with vocabulary size. \Cref{fig:tag_vocab_scatter} plots $\log|\mathcal{V}_\tau|$, the log number of distinct target-token types realized under tag $\tau$, against raw and adjusted effects. Large-vocabulary, open-class tags tend to be more hybrid-favored in both views. The adjusted slope is modest, as frequency, difficulty, and reuse explain some of the same variation, but remains positive, consistent with the aggregate content--function result. Several closed-class/function-like tags still have strongly negative adjusted effects despite the weaker vocabulary-size relationship. An exception is \textsc{Existential}, which remains unusually hybrid-favored despite a small choice set, including after controls. Plausibly, existential \emph{there} often introduces discourse (``there is/are $\ldots$''), so its selection can depend on latent state updates despite the small surface choice set.

\begin{figure}[t]
\centering
\includegraphics[width=\linewidth]{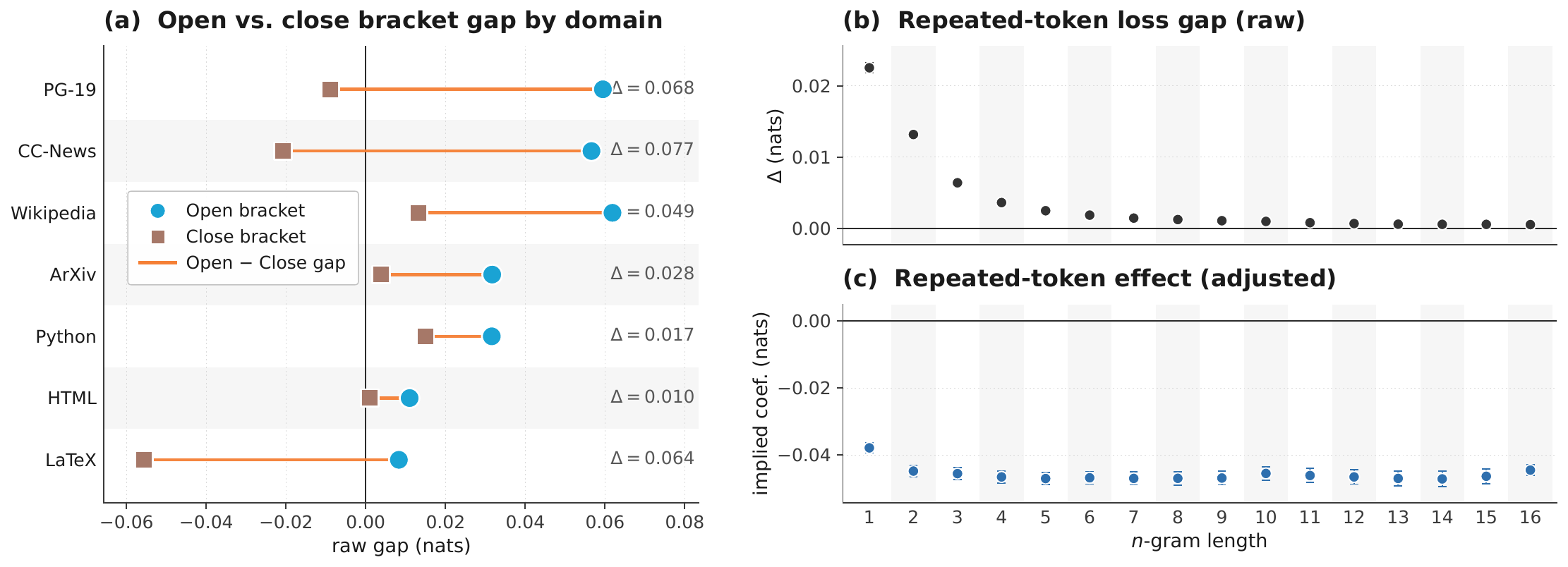}
\caption{\textbf{Brackets and repeated $n$-grams.}
\textbf{A:} Open and close bracket raw gaps across domains; line segments show the open--close difference.
\textbf{B:} Raw paired loss gaps for repeated-token events from repeated 1-grams through 16-grams. Positive raw gaps mean lower hybrid NLL.
\textbf{C:} Implied repetition effects from an extension of Equation~\ref{eq:prose_reg} with \textsc{copy}$_1$--\textsc{copy}$_{16}$. Negative adjusted effects mean that repetition reduces the hybrid advantage relative to comparable non-copy positions.}
\vspace{-8pt}
\label{fig:bracket_copy_gap}
\end{figure}

\textbf{Opening vs.\ closing brackets: the hybrid advantage is reduced on structural closure. } If the bracket pattern reflects predictive role rather than character, openers and closers should differ: openers often initiate a region/scope (a state update), whereas closers typically satisfy a structural obligation established by the visible prefix (closure). \Cref{fig:bracket_copy_gap}A tests this prediction. Across all seven domains (prose, Python, HTML, \LaTeX{}), \emph{opening} brackets are consistently more hybrid-favored than corresponding \emph{closing} brackets; the gap is largest in prose and remains visible in structured text. Because their surface forms are similar, this cross-domain consistency supports an interpretation based on the computation queried at the target.

\textbf{Hybrid advantage nearly disappears on repeated $n$-grams.} \Cref{fig:bracket_copy_gap}B--C isolates repeated spans, a weak observable proxy for prefix-based copying, not direct evidence of a copying mechanism. A repeated $n$-gram event occurs when the contiguous $n$-token sequence ending at the target appeared earlier in the prefix. For example, in ``\texttt{a b c d a b}\,\underline{\texttt{c}}'', the target completes the repeated 3-gram ``\texttt{a b c}''. \Cref{fig:bracket_copy_gap}B reports raw paired gaps for $n=1,\ldots,16$. Repeated 1-grams capture ordinary token-type reuse, rather than clean repeated-span continuation. Larger $n$ gives a cleaner proxy, as longer repeated spans increasingly determine the target through a visible continuation. Consistently, the raw hybrid advantage shrinks rapidly with $n$, approaching zero for long spans, where this strong prefix-based prediction leaves no measurable absolute advantage.

The regression view in \Cref{fig:bracket_copy_gap}C makes the same point under controls.
We extend Equation~\ref{eq:prose_reg} by replacing \textsc{copy}$_1$--\textsc{copy}$_4$ with \textsc{copy}$_1$--\textsc{copy}$_{16}$.
Because repeated-$n$ events nest (a repeated $n$-gram also activates all shorter copy indicators), we plot the \emph{implied} repetition effect at length $n$ by summing active copy coefficients through $n$.
These implied effects are consistently transformer-shifting (negative), with uncertainty growing for large $n$ because long repeated spans are rare and the nested indicators are highly correlated.
Non-repetition control coefficients are reported in Appendix~\ref{app:control_coefficients}.

\subsection{Controlled synthetic analysis}
\label{sec:synthetic_results}

\begin{figure}[t]
\centering
\includegraphics[width=\linewidth]{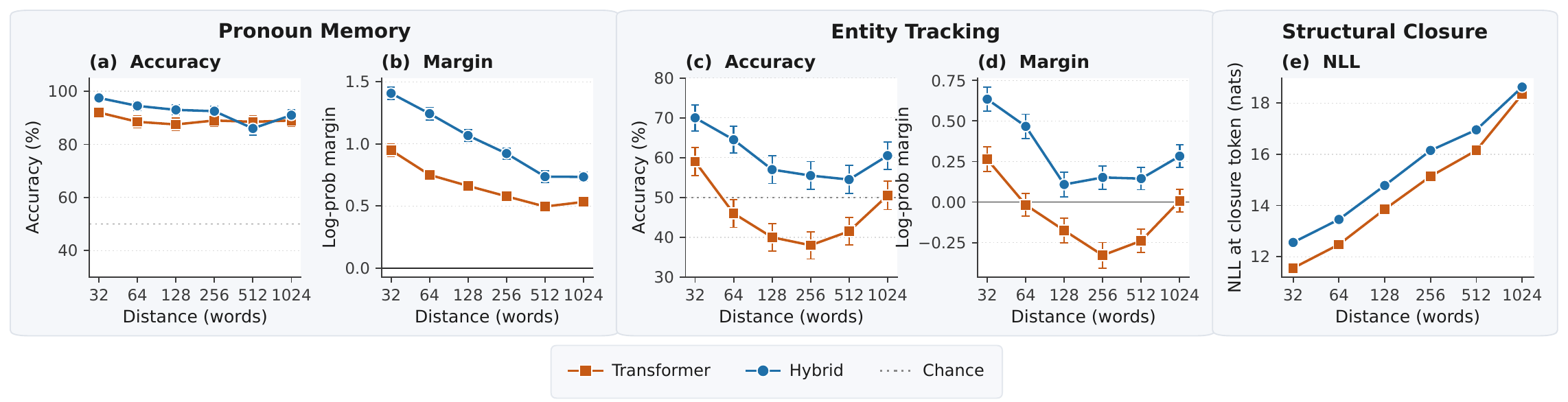}
\caption{\textbf{Controlled probes at the final matched checkpoint.}
Pronoun memory and entity tracking are scored by contrastive accuracy and log-probability margin; structural closure is scored by NLL on the closing token.
Higher is better for accuracy and margin; lower is better for NLL.}
\vspace{-10pt}
\label{fig:distance_row}
\end{figure}

\Cref{fig:distance_row} tests the natural-token split with fixed templates and varying dependency distances and target types. Pronoun memory favors the hybrid, with larger contrastive margins and typically higher accuracy. Entity tracking separates more strongly: the hybrid stays above chance across distances; the transformer dips below chance with negative margins at intermediate distances. Structural closure reverses sign: the transformer has lower NLL at every distance on closers with visible openers. All three probes require earlier context, but only state-readout probes favor the hybrid, so delay alone is not decisive. This agrees with the theoretical expectation that recurrence helps most when maintaining and querying latent state, while attention remains competitive or better on closure determined by visible openers.

\subsection{Summary of empirical findings}

Observational tagging and controlled probes reveal a consistent, non-uniform hybrid--transformer gap: (i) broad hybrid gains are strongest on open-class, content-bearing choices, and weaker and more heterogeneous on closed-class/function categories; (ii) the transformer matches or exceeds the hybrid on structural closure: closers are less hybrid-favored than openers across domains, and the controlled closure probe favors the transformer; (iii) hybrid gains collapse on explicit visible-prefix reuse (repeated $n$-grams), and repetition features shift the gap toward the transformer even after controls.

\section{\upd{Speculative Explanations for Token-Level Architecture Differences}}
\label{sec:discussion}

We now \upd{make some basic theoretical observations related to the empirical findings} in \S\ref{sec:empirical_results}.
\upd{
Suppose we have a language model trained on the target distribution $p^*(X_t \mid X_{<t})$.
%
The model implicitly learns some features $\phi(X_{<t})$ that are used to predict the next token $X_t$. Crucially, the prediction is conditionally independent from $X_{<t}$ given $\phi(X_{<t})$ and folllows the distribution
\begin{equation*}
    p_\phi(\hat X_t \mid \phi(X_{<t})) .
\end{equation*}
Under our idealization, we interpret $\phi(X_{<t})$ as a vector of features summarizing the prefix $X_t$ that are representable by the architecture.
For example, attention can expose visible-prefix retrieval and structural-matching features, while recurrent layers can expose updated latent-state features.
We now make some basic observations about how the approximation error of $p_\phi$ with respect to $p^*$ depends on the set of features $\phi$ exposed by the architecture and the properties of the token type $\tau$.
Formally, we define approximation error in terms of the KL divergence on tokens of type $\tau$:
\begin{align*}
    D_{\mathrm{KL}}^\tau\!\left(p^\star\,\|\,p_{\phi}\right)
    &= \underset{x \sim p^*_\tau}{\mathbb{E}} \left[ \log p^*(x) - \log p_{\phi}(x) \mid x \in \mathcal V_\tau \right] ,
\end{align*}
where $\mathcal V_t$ is the set of tokens that can realize type $\tau$.
A feature map is predictive enough for a token family when this class-conditional KL divergence is small on that family.
}

\upd{
\textbf{Why might attention suffice on recall and bracket closure?}
We make a simple observation: if the features an architecture exposes reveal all information relevant for predicting tokens of type $\tau$, the optimal error for predicting the tokens of type $\tau$ will be zero (cf.~\Cref{app:theory-proofs} for a proof):
\begin{restatable}{proposition}{zeroError}
    If 
    $X_t$ is a deterministic function of $\phi(X_{<t})$ on tokens of type $\tau$, 
    then $D_{\mathrm{KL}}^\tau(p^\star\|p_{\phi})=0$.
\end{restatable}
}
In less deterministic natural text, the corresponding claim is that this class-conditional KL is small.  Thus, once such a feature is available, a richer state representation has little additional loss to remove.
In the repeated-$n$-gram case, $\phi$ is the visible continuation of the longest repeated suffix: if the current suffix has occurred before and its continuation is unambiguous, then the target is predictable by continuing material already present in the prefix.  In the closure case, $\phi$ is the relevant structural matcher, such as the top of the active delimiter stack: if the target is a closing delimiter, the correct closer is predictable from that stack state.  Since both the matched transformer and the hybrid can express these features in the regimes measured by repeated continuations and bracket/tag closure, we expect little systematic difference between them there.
\upd{This provides a potential explanation for why GDN layers do not significantly improve loss on repeated long $n$-grams and closing delimiters: attention already provides the right features to disambiguate the next token, with or without GDN.}

\upd{
\textbf{Why might more expressive architectures benefit larger token classes?}
We next consider how changing the features $\phi$ representable by the architecture could improve loss. As shown in \Cref{app:theory-proofs},}
\begin{restatable}{proposition} {klUpperBound}\label{prop:vocab-size}
    \upd{For any token type $\tau$, we have}
    $D_{\mathrm{KL}}^\tau(p^\star\|p_{\phi}) \le \log |\mathcal{V}_{\tau}| $.
\end{restatable}

This recalls the POS-vocabulary-size effect in \Cref{fig:tag_vocab_scatter}.
For small closed classes, knowing the local class leaves only a limited amount of possible extra information.  Thus, the additional benefit of information beyond the local class is bounded even when that information is useful.  For open classes, the bound is much looser: knowing that the next token is a noun, identifier, string, or text node can still leave substantial uncertainty about which token should fill the slot.  This prediction is consistent with our empirical finding that many function-like categories have smaller and more heterogeneous hybrid gains, while larger open-class categories leave more room for architectural differences.

\upd{\textbf{Why might GDN be helping open-class categories?}
We do not know, but we hypothesize that GDN could allow representing features $\phi(X_{<t})$ of the prefix useful for \emph{discourse state tracking}.
In linguistics, discourse refers to semantic context that evolves as a document is processed. The meaning of a sentence can be conceptualized as an update operator to the discourse state: in particular, tracking the state of entities that have been introduced plays a central role \citep{heim1982semantics,heim1983file,kamp1993discourse,grosz1986attention,grosz1995centering}.
As schematized in \Cref{fig:discourse-a5-bridge}, features about the state of entities in the discourse could be helpful in predicting open-class words where semantic information is important for sharpening the prediction.
\Cref{prop:vocab-size} suggests opportunity for this effect on open-class token types, but minimal impact on closed-class token types.
}

\upd{This hypothesis is inspired by the state-tracking abilities of GDN, which are often benchmarked using the} $A_5$ word problem \citep{merrill2024illusion,grazzi2025unlocking,peng2025rwkv,yang2026path,merrill2026olmohybrid}.
This task requires models to compose sequential updates to a fixed-size array of 5 variables, capturing one notion of hard state tracking.
Unlike $A_5$, natural text and code can introduce new entities, events, variables, scopes, and situations over time, so the effective state size can grow with input length. The state is also relational: useful readouts may depend on who did what to whom, which object is salient, which scope is active, or which variable binding currently holds \citep{kamp1993discourse}. Thus, more realistic state tracking could be a good synthetic-task target for future architectures. Such tasks could extend $A_5$ with dynamic entity introduction and relational state, making them a more demanding testbed for future linear-RNN and hybrid sequence layers.

\definecolor{pm}{HTML}{1AA3D4}
\definecolor{pmsoft}{HTML}{F4FAFD}      
\definecolor{pmedge}{HTML}{B5DCEB}      
\definecolor{et}{HTML}{A67868}
\definecolor{etsoft}{HTML}{FBF6F2}      
\definecolor{etedge}{HTML}{DCC1AC}      
\definecolor{ink}{HTML}{2A2A2A}
\definecolor{muted}{HTML}{6F6F6F}
\definecolor{bridgegrey}{HTML}{8A8A8A}  

\tikzset{
  >=Stealth,
  panelP/.style={
    rounded corners=8pt,
    line width=.85pt,
    fill=pmsoft,
    draw=pmedge
  },
  panelE/.style={
    rounded corners=8pt,
    line width=.85pt,
    fill=etsoft,
    draw=etedge
  },
  ptitle/.style={
    font=\bfseries\small,
    text=pm,
    anchor=north west,
    inner xsep=0pt,
    inner ysep=0pt
  },
  etitle/.style={
    font=\bfseries\small,
    text=et,
    anchor=north west,
    inner xsep=0pt,
    inner ysep=0pt
  },
  sent/.style={
    font=\footnotesize,
    text=ink,
    align=left,
    anchor=north west,
    inner xsep=0pt,
    inner ysep=0pt
  },
  cloze/.style={
    font=\scriptsize\itshape,
    text=pm,
    anchor=north east,
    inner xsep=0pt,
    inner ysep=0pt
  },
  pnode/.style={
    draw=pm!75,
    fill=white,
    rounded corners=2.5pt,
    line width=.7pt,
    inner xsep=4pt,
    inner ysep=2.5pt,
    font=\footnotesize,
    text=ink,
    align=center,
    minimum height=5mm,
    minimum width=12mm
  },
  goalnode/.style={
    draw=pm,
    fill=pm!10,
    rounded corners=2.5pt,
    line width=.95pt,
    inner xsep=4pt,
    inner ysep=2.5pt,
    font=\footnotesize\itshape,
    text=pm!75!black,
    align=center,
    minimum height=5mm,
    minimum width=12mm
  },
  rarrow/.style={
    draw=muted,
    line width=.65pt,
    -{Stealth[length=1.6mm,width=1.3mm]}
  },
  relarrow/.style={
    draw=pm,
    line width=.9pt,
    dash pattern=on 1.9pt off 1.5pt,
    -{Stealth[length=1.8mm,width=1.5mm]}
  },
  elabel/.style={
    font=\scriptsize\itshape,
    text=muted,
    fill=pmsoft,
    fill opacity=.96,
    text opacity=1,
    inner xsep=1.5pt,
    inner ysep=.4pt
  },
  rellabel/.style={
    font=\scriptsize\itshape,
    text=pm,
    fill=pmsoft,
    fill opacity=.96,
    text opacity=1,
    inner xsep=1.5pt,
    inner ysep=.4pt
  },
  codebox/.style={
    draw=etedge,
    fill=white,
    rounded corners=5pt,
    line width=.7pt,
    inner xsep=7pt,
    inner ysep=6pt,
    font=\ttfamily\footnotesize,
    text=ink,
    align=left
  },
  bridge/.style={
    draw=bridgegrey,
    fill=white,
    rounded corners=3pt,
    line width=.75pt,
    inner xsep=3.5pt,
    inner ysep=2.5pt,
    font=\scriptsize\itshape,
    text=bridgegrey,
    align=center
  },
  bridgearr/.style={
    draw=bridgegrey,
    line width=.8pt,
    -{Stealth[length=1.7mm,width=1.4mm]}
  }
}

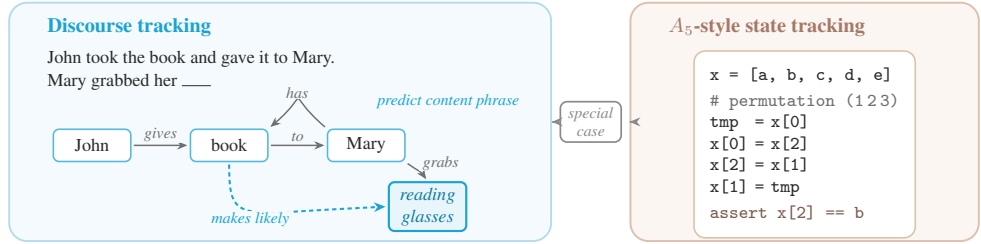
\begin{figure*}[t]
\centering
\resizebox{.92\textwidth}{!}{%
\begin{tikzpicture}[font=\small]

\def\PW{8.4cm}
\def\AW{5.4cm}
\def\PH{3.7cm}
\def\Gap{1.2cm}

\node[panelP, minimum width=\PW, minimum height=\PH, anchor=north west] (D) at (0,0) {};
\node[panelE, minimum width=\AW, minimum height=\PH, anchor=north west] (A)
  at ([xshift=\Gap]D.north east) {};

\node[ptitle] (dt)
  at ([xshift=6mm, yshift=-2.5mm]D.north west)
  {Discourse tracking};

\node[sent] (s1)
  at ([yshift=-2mm]dt.south west)
  {John took the book and gave it to Mary.};

\node[sent] (s2)
  at ([yshift=-1mm]s1.south west)
  {Mary grabbed her \rule{4.5mm}{.35pt}};

\node[cloze]
  at ([xshift=-5mm, yshift=-0.5mm]s2.south -| D.east)
  {predict content phrase};

\node[pnode] (john)
  at ([xshift=13mm, yshift=15mm]D.south west)
  {John};

\node[pnode, right=9mm of john] (book)
  {book};

\node[pnode, right=9mm of book] (mary)
  {Mary};

\node[goalnode] (glasses)
  at ([xshift=65mm, yshift=5.5mm]D.south west)
  {reading\\glasses};

\draw[rarrow, shorten <=1pt, shorten >=1pt]
  (john.east) --
  node[elabel, above=.35mm] {gives}
  (book.west);

\draw[rarrow, shorten <=1pt, shorten >=1pt]
  (book.east) --
  node[elabel, above=.35mm] {to}
  (mary.west);

\draw[rarrow, shorten <=1pt, shorten >=1pt]
  (mary.north west)
  .. controls +(-7mm,6mm) and +(7mm,6mm) ..
  node[elabel, above, pos=.52] {has}
  (book.north east);

\draw[relarrow, shorten <=1pt, shorten >=1pt]
  (book.south)
  .. controls +(0mm,-10mm) and +(-23mm,-1mm) ..
  node[rellabel, below, pos=.48] {makes likely}
  (glasses.west);

\draw[rarrow, shorten <=1pt, shorten >=1pt]
  (mary.south east)
  to[out=-35,in=125]
  node[elabel, above right, pos=.55] {grabs}
  (glasses.north);

\node[etitle] (atit)
  at ([xshift=6mm, yshift=-2.5mm]A.north west)
  {$A_5$-style state tracking};

\node[codebox, anchor=center]
  at ($(A.center)+(0,-3.5mm)$) (code)
  {\begin{tabular}{@{}l@{\,\,}l@{\,\,}l@{}}
   \multicolumn{3}{@{}l@{}}{x = [a, b, c, d, e]}\\[2pt]
   \multicolumn{3}{@{}l@{}}{\textcolor{muted}{\# permutation (1\,2\,3)}}\\
   tmp  & = & x[0]\\
   x[0] & = & x[2]\\
   x[2] & = & x[1]\\
   x[1] & = & tmp\\[2pt]
   \multicolumn{3}{@{}l@{}}{\textcolor{et!75!black}{assert x[2] == \rule{4.5mm}{.35pt}}}
   \end{tabular}};

\node[bridge] (br)
  at ($(D.east)!0.5!(A.west)$)
  {special\\case};

\draw[bridgearr]
  (D.east) -- ([xshift=-1.2mm]br.west);

\draw[bridgearr]
  ([xshift=1.2mm]br.east) -- (A.west);

\end{tikzpicture}%
}
\caption{\textbf{Discourse tracking as richer state tracking.}
\textbf{Left:} predicting the masked content phrase requires tracking that the book has ended up with Mary and using that discourse state to make reading-related objects likely.
\textbf{Right:} closed-world $A_5$ permutation updates apply ordered, deterministic state changes over a fixed set of variables, followed by an assertion about one variables.}
\vspace{-10pt}
\label{fig:discourse-a5-bridge}
\end{figure*}

\section{Application: Higher-Signal Evaluations for Pretraining Hybrid Models}
\label{sec:higher_signal_evals}

The decomposition in \S\ref{sec:empirical_results} suggests that aggregate next-token loss mixes computation regimes that respond differently to the sequence mixer. Most tokens are easy for any reasonable architecture, making aggregate loss a relatively \emph{low-signal} architecture-search target: the easy majority can dilute real, systematic gaps. We propose \emph{filtered} losses from the same next-token NLL as a higher-signal complement. By focusing on tokens where a capability matters, filters can amplify small architecture gaps and reveal behavior across token families hidden by an aggregate number. Our proof of concept shows capability differences nearly invisible in aggregate validation loss.

\textbf{Filters.} We evaluate three filters built from the token-level features of \S\ref{sec:empirical_method}:
(i) \textsc{All tokens}, the standard aggregate;
(ii) \textsc{Top-10}\,$\cap$\,\textsc{No-Copy}, restricted to the ten most hybrid-favored open-class POS families (\Cref{fig:tag_raw_vs_effect}) and excluding positions that complete a repeated $n$-gram for $n\le4$;
(iii) \textsc{Copy-5\,only}, positions that complete a repeated 5-gram.
By construction, (ii) targets \emph{state-conditioned readout} while removing visible-prefix retrieval, and (iii) isolates retrieval.

\textbf{Setup. }
To test their architectural signal during small-scale pretraining, we use three released 1B-parameter development runs from \citet{merrill2026olmohybrid}: a \textbf{Transformer}, a \textbf{Hybrid} (interleaved GDN/attention, $3{:}1$ ratio), and a \textbf{Pure RNN} (GDN, no attention). They share a data mixture and budget under WSD scheduling. We score every annealed checkpoint; \Cref{fig:filtered_eval_1b} shows token-loss curves.\upd{\footnote{These 1B runs provide a matched pure-RNN comparison unavailable at 7B; see Appendix~\ref{app:filtered_eval_7b} for 7B results.}}

\textbf{State-oriented filters amplify small architecture gaps. }
Under aggregate \textsc{All tokens} loss (\Cref{fig:filtered_eval_1b}, left), the Transformer--Hybrid separation is small, with a maximum of roughly $0.06$ nats and smaller gaps later in training.
Under \textsc{Top-10}\,$\cap$\,\textsc{No-Copy} (middle), the maximum Transformer--Hybrid separation is roughly $0.12$ nats, about twice as large, and the ordering becomes \textsc{Hybrid}\,$<$\,\textsc{Pure RNN}\,$<$\,\textsc{Transformer}.
This supports the intended use of the filter: removing copy positions and restricting to open-class targets leaves a regime that is more sensitive to how well an architecture supports latent-state construction and state-conditioned readout.

\textbf{Copy filters expose what aggregate loss averages away.}
Under aggregate \textsc{All Tokens} loss, the Transformer and Pure RNN appear roughly matched, especially at later checkpoints, which, taken alone, suggests similar validation behavior. Filtered curves reveal that this match averages different regimes: the Pure RNN outperforms the Transformer on state-oriented non-copy targets under \textsc{Top-10} $\cap$ \textsc{No-Copy}, but has a compensating weakness under \upd{\textsc{Copy-5\,only}}. On repeated 5-gram continuations, where prediction mainly requires visible-prefix retrieval, it is consistently about $0.10$--$0.20$ nats worse than both attention-based models. Thus recurrence-only models can appear competitive on average while lagging on retrieval, a split aggregate loss largely obscures.

\textbf{A diagnostic for hybrid architecture design.}
Aggregate loss reports overall competitiveness; filtered evals diagnose \emph{why}. \upd{We recommend the joint profile of \textsc{All Tokens}, \textsc{Top-10}$\,\cap\,$\textsc{No-Copy}, and \textsc{Copy-5\,only} for architecture selection: a candidate should not be selected if its filtered-loss improvement comes with worse visible-prefix copying or structural closure.} A Pure RNN that closes the \textsc{Top-10}\,$\cap$\,\textsc{No-Copy} gap but remains far behind attention-based models on \textsc{Copy-5\,only} improves state representation without solving retrieval; the reverse borrows attention's strength without improving state tracking. For \emph{hybrid} designs spanning layer ratios, recurrent mixers, and attention placement, this decomposition lets us ask which choices improve which capability, beyond a single aggregate over both regimes. Both filters use standard validation's per-token NLL at negligible additional cost; we suggest reporting them alongside aggregate loss during architecture sweeps.

\begin{figure}[t]
\centering
\includegraphics[width=0.9\linewidth]{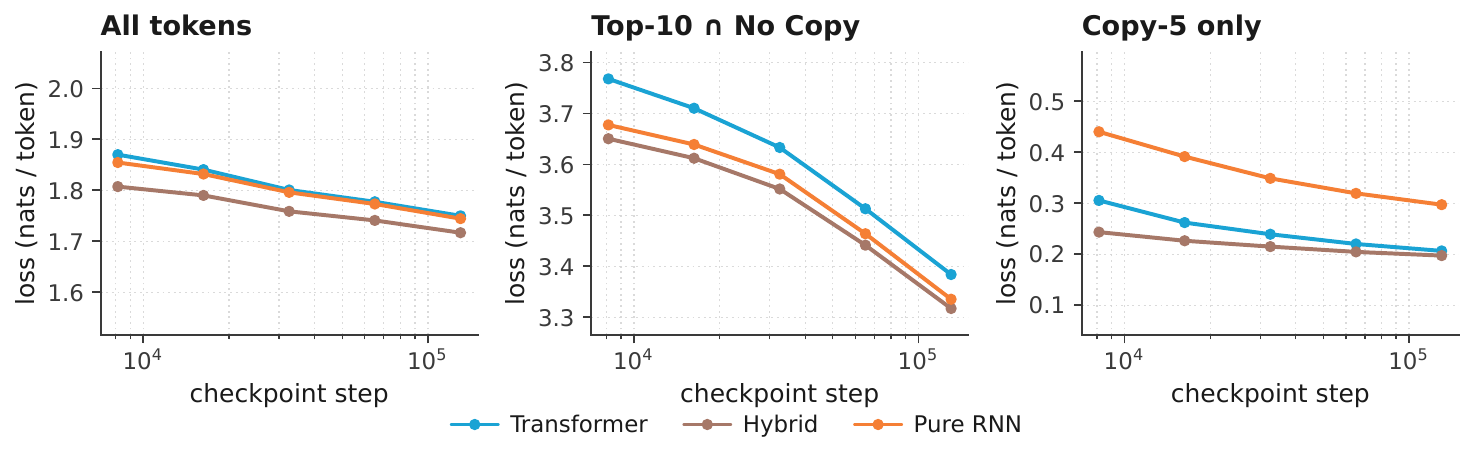}
\caption{\textbf{Filtered token losses surface architecture differences during 1B
pretraining.} Token-loss curves at WSD-annealed checkpoints for a Transformer, a
Hybrid, and a Pure RNN.
}
\label{fig:filtered_eval_1b}
\end{figure}

\section{Conclusion}
We compared a matched transformer (Olmo 3) and hybrid (Olmo Hybrid) at the token level to identify predictions driving their average loss gap. Three patterns recur: (i) the hybrid's edge concentrates on open-class content predictions; (ii) opening delimiters are hybrid-favored and closing delimiters transformer-favored; (iii) the hybrid advantage vanishes on visible-prefix copying. These findings seem consistent with hybrids' complementary expressivity: attention provides direct visible-prefix retrieval, and recurrent layers support ordered state update and state-conditioned readout. This plausibly explains why hybrid gains concentrate on predictions depending on evolving discourse, program, or document state rather than copying or bracket closure.

Two directions follow. First, $A_5$-style benchmarks capture only a closed-world special case of natural-language state tracking. Formal-language tasks with dynamic entity introduction and relational structure, evaluated on linear-RNN variants (\Cref{sec:discussion}), would test whether architectures with better discourse-style state tracking improve on the token families driving hybrid gains. Second, token-level profiles can diagnose data-mixture and post-training choices: underperformance on narrow families (e.g., numeric literals, closing delimiters) suggests where targeted data might help, while the remaining profile checks for erosion of state-dependent predictions. Our filtered evaluations offer an encouraging proof of concept: token-level decompositions can sharpen pretraining comparisons, expose capability tradeoffs hidden by aggregate loss, and identify token families for targeted intervention. Future work can explore their practical implications for architecture search, data selection, and post-training.

\begin{ack}
The authors thank Allyson Ettinger and Kyle Richardson for relevant feedback.
\end{ack}

\bibliographystyle{plainnat}
\bibliography{references}


\newpage
\appendix

\section{Empirical methodology details}
\label{app:method_details}

\subsection{Domains, packing, and pairing protocol}
\label{app:domains_details}

\upd{For the primary Olmo comparison, we evaluate both models} on the same evaluation sequences and compute NLL at every position. Text is packed into contiguous sequences of length $T=8192$. Allcomparisons are \emph{paired} at the level of a single next-token decision: \emph{same checkpoint, same prefix/context, same target token}. For each position $i$ we compute $\Delta_i=\ell^{\mathrm{T}}_i-\ell^{\mathrm{H}}_i$.

For the main text we focus on prose (PG-19, CC-News, Wikipedia, ArXiv) and three structured domains (Python, HTML, \LaTeX{}). The pooled prose regression uses a slightly expanded prose suite at the final matched checkpoint (including additional textbooks and scientific papers), but uses the same packing and pairing protocol.

\subsection{Token tagging and alignment}
\label{app:tagging_details}

\paragraph{Prose POS tags and content/function mapping.}
We POS-tag prose at the \emph{word} level using the Brown tagset
\citep{francis1979brown,bird2009nltk}, then map each Brown tag to a coarse family.
For the aggregate prose analysis, we define:

\begin{itemize}[leftmargin=1.2em,itemsep=0.2em,topsep=0.2em]
  \item \textbf{Content words:}
  \textsc{Noun}, \textsc{Verb}, \textsc{Adjective}, \textsc{Adverb},
  \textsc{Interjection}, \textsc{Qualifier}.
  \item \textbf{Function words:}
  \textsc{Existential}, \textsc{Pronoun}, \textsc{Det/Article},
  \textsc{Preposition}, \textsc{Conjunction}, \textsc{Aux BE},
  \textsc{Aux HAVE}, \textsc{Aux DO}, \textsc{Modal}, \textsc{TO},
  \textsc{Wh-word}.
  \item \textbf{Other:} punctuation, brackets, numerals, particles, and residual
  tags.
\end{itemize}

\paragraph{Coarse Brown-tag mapping.}
Table~\ref{tab:coarse_mapping} gives the exact mapping used to produce the coarse
families in the main-text prose figures.

\begin{table}[H]
  \centering
  \caption{Mapping from Brown tags to the coarse categories used in the prose analyses.}
  \label{tab:coarse_mapping}
  \small
  \begin{tabular}{lp{9.5cm}}
    \toprule
    \textbf{Coarse category} & \textbf{Brown tags} \\
    \midrule
    \multicolumn{2}{l}{\emph{Content words}} \\
    Noun        & \texttt{NN}, \texttt{NN\$}, \texttt{NNS}, \texttt{NNS\$}, \texttt{NP}, \texttt{NP\$}, \texttt{NPS}, \texttt{NPS\$}, \texttt{NR}, \texttt{NR\$}, \texttt{NRS} \\
    Verb        & \texttt{VB}, \texttt{VBD}, \texttt{VBG}, \texttt{VBN}, \texttt{VBZ} \\
    Adjective   & \texttt{JJ}, \texttt{JJ\$}, \texttt{JJR}, \texttt{JJS}, \texttt{JJT} \\
    Adverb      & \texttt{RB}, \texttt{RB\$}, \texttt{RBR}, \texttt{RBT}, \texttt{RN} \\
    Interjection & \texttt{UH} \\
    Qualifier   & \texttt{QL}, \texttt{QLP} \\
    \midrule
    \multicolumn{2}{l}{\emph{Function words}} \\
    Pronoun     & \texttt{PP\$}, \texttt{PP\$\$}, \texttt{PPL}, \texttt{PPLS}, \texttt{PPO}, \texttt{PPS}, \texttt{PPSS}, \texttt{PN}, \texttt{PN\$} \\
    Det/Article & \texttt{AT}, \texttt{DT}, \texttt{DT\$}, \texttt{DTI}, \texttt{DTS}, \texttt{DTX}, \texttt{ABL}, \texttt{ABN}, \texttt{ABX}, \texttt{AP}, \texttt{AP\$} \\
    Preposition & \texttt{IN} \\
    Conjunction & \texttt{CC}, \texttt{CS} \\
    Aux BE      & \texttt{BE}, \texttt{BED}, \texttt{BEDZ}, \texttt{BEG}, \texttt{BEM}, \texttt{BEN}, \texttt{BER}, \texttt{BEZ} \\
    Aux HAVE    & \texttt{HV}, \texttt{HVD}, \texttt{HVG}, \texttt{HVN}, \texttt{HVZ} \\
    Aux DO      & \texttt{DO}, \texttt{DOD}, \texttt{DOZ} \\
    Modal       & \texttt{MD} \\
    TO          & \texttt{TO} \\
    Existential & \texttt{EX} \\
    \midrule
    \multicolumn{2}{l}{\emph{Other categories}} \\
    Wh-word     & \texttt{WDT}, \texttt{WP\$}, \texttt{WPO}, \texttt{WPS}, \texttt{WQL}, \texttt{WRB} \\
    Numeral     & \texttt{CD}, \texttt{CD\$}, \texttt{OD} \\
    Particle    & \texttt{RP} \\
    Punctuation & \texttt{,}, \texttt{.}, \texttt{:}, \texttt{'}, \texttt{''}, \texttt{``}, \texttt{--}, \texttt{*} \\
    Open Bracket & \texttt{(} \\
    Close Bracket & \texttt{)} \\
    \bottomrule
  \end{tabular}
\end{table}

\paragraph{Projecting word tags onto subword tokens.}
Because the LM uses subword tokenization, we assign each word-level tag to every
LM token whose decoded character span overlaps the tagged word. Each LM token
also receives a word-position indicator \textsc{Whole}/\textsc{Prefix}/%
\textsc{Middle}/\textsc{Suffix} to separate category effects from word-onset
effects.

\paragraph{Structured domains and multi-tag overlap.}
For Python we tokenize files using the standard \texttt{tokenize} module. For
HTML and \LaTeX{} we use lightweight parser/regex pipelines that separate
text-like content from structural tokens (tags/delimiters/commands/whitespace).
We align source-level tags to LM subword tokens by character-span overlap.

A single LM token can overlap \emph{multiple} source tokens (e.g., \texttt{):}
spans both a close-paren and a colon). We therefore use \textbf{multi-tag
attribution}: the LM-token loss contributes to every overlapping source tag
rather than forcing a single primary label.




\section{Structured-domain tagsets and coarse breakdowns}
\label{app:structured_domains}

This section supplies the structured-domain tag taxonomies.

\subsection{Python}
\label{sec:python_analysis}

\paragraph{Tagging and multi-tag alignment.}
We tokenize source files with Python's \texttt{tokenize} module and map the
resulting source-level tags onto LM subword tokens by character-span overlap.
Because one LM token can contain several Python tokens (e.g., \texttt{):}), the
Python analysis uses multi-tag attribution.

\begin{table}[H]
  \centering
  \footnotesize
  \caption{Python token taxonomy.}
  \label{tab:python_tags}
  \begin{tabularx}{\linewidth}{>{\raggedright\arraybackslash}p{1.45cm}>{\raggedright\arraybackslash}p{3.8cm}X}
    \toprule
    Coarse & Fine-grained tags & Examples \\
    \midrule
    Identifier & \texttt{identifier}, \texttt{upper\_camel}, \texttt{private}, \texttt{dunder} & \texttt{x}, \texttt{MyClass}, \texttt{\_val}, \texttt{\_\_init\_\_} \\
    Keyword    & \texttt{keyword}, \texttt{soft\_keyword} & \texttt{def}, \texttt{return}, \texttt{match} \\
    Builtin    & \texttt{builtin} & \texttt{print}, \texttt{len}, \texttt{ValueError} \\
    Operator   & \texttt{arithmetic}, \texttt{assignment}, \texttt{comparison}, \texttt{bitwise} & \texttt{+}, \texttt{+=}, \texttt{==}, \texttt{\&} \\
    Delimiter  & \texttt{enclosure}, \texttt{delimiter}, \texttt{other} & \texttt{(}, \texttt{)}, \texttt{,}, \texttt{:}, \texttt{.} \\
    String     & \texttt{string}, \texttt{triple\_quoted}, \texttt{fstring}, \texttt{raw}, \texttt{bytes} & \texttt{'hello'}, \texttt{"""doc"""}, \texttt{f'\{x\}'} \\
    Number     & \texttt{integer}, \texttt{float}, \texttt{hex}, \texttt{octal}, \texttt{complex} & \texttt{42}, \texttt{3.14}, \texttt{0xFF} \\
    Comment    & \texttt{COMMENT} & \texttt{\# ...} \\
    Structure  & \texttt{NEWLINE}, \texttt{INDENT}, \texttt{WHITESPACE} & indentation and layout \\
    \bottomrule
  \end{tabularx}
\end{table}

\FloatBarrier

\subsection{HTML} \label{sec:html_analysis}

We tag HTML with a parser-plus-regex pipeline that distinguishes opening/closing
markup, attributes, text nodes, comments, punctuation glyphs, and whitespace.

\begin{table}[H]
  \centering
  \footnotesize
  \caption{HTML token taxonomy.}
  \label{tab:html_tags}
  \begin{tabularx}{\linewidth}{>{\raggedright\arraybackslash}p{1.35cm}>{\raggedright\arraybackslash}p{3.7cm}X}
    \toprule
    Coarse & Fine-grained tags & Examples \\
    \midrule
    Tag         & \texttt{open}, \texttt{close}, \texttt{self\_close}, \texttt{void}, \texttt{doctype} & \texttt{<div>}, \texttt{</div>}, \texttt{<br/>}, \texttt{<img>} \\
    Attribute   & \texttt{name}, \texttt{equals}, \texttt{value\_quoted}, \texttt{value\_unquoted} & \texttt{class}, \texttt{=}, \texttt{"main"} \\
    Text        & \texttt{content}, \texttt{entity} & free-form page text \\
    Comment     & \texttt{comment} & \texttt{<!-- ... -->} \\
    Punctuation & \texttt{angle\_open}, \texttt{angle\_close}, \texttt{slash}, \texttt{quote} & \texttt{<}, \texttt{>}, \texttt{/}, \texttt{"} \\
    Whitespace  & \texttt{ws} & spaces and indentation \\
    \bottomrule
  \end{tabularx}
\end{table}

\FloatBarrier

\subsection{\LaTeX{}}

We tag \LaTeX{} with a lightweight regex pipeline that separates commands,
environments, math, tables, grouping symbols, text, comments, and newlines.

\begin{table}[H]
  \centering
  \footnotesize
  \caption{\LaTeX{} token taxonomy.}
  \label{tab:latex_tags}
  \begin{tabularx}{\linewidth}{>{\raggedright\arraybackslash}p{1.45cm}>{\raggedright\arraybackslash}p{3.9cm}X}
    \toprule
    Coarse & Fine-grained tags & Examples \\
    \midrule
    Command     & \texttt{control}, \texttt{section}, \texttt{ref}, \texttt{cite}, \texttt{formatting}, etc. & \texttt{\textbackslash section}, \texttt{\textbackslash textbf} \\
    Environment & \texttt{begin/end}, \texttt{begin/end\_math}, \texttt{begin/end\_table} & \texttt{\textbackslash begin\{table\}} \\
    Math        & \texttt{inline/display\_open/close}, \texttt{env\_name} & \texttt{\$}, \texttt{\$\$} \\
    Table       & \texttt{ampersand}, \texttt{hline}, \texttt{rule}, \texttt{multicolumn}, \texttt{linebreak} & \texttt{\&}, \texttt{\textbackslash\textbackslash} \\
    Group       & \texttt{brace/bracket\_open/close} & \texttt{\{}, \texttt{\}} \\
    Text        & \texttt{content}, \texttt{space} & ordinary words and spaces \\
    Comment     & \texttt{comment} & \texttt{\% ...} \\
    Special     & \texttt{escaped}, \texttt{tilde}, \texttt{caret}, \texttt{underscore} & escaped symbols \\
    Newline     & \texttt{newline}, \texttt{blank\_line} & line-level layout \\
    \bottomrule
  \end{tabularx}
\end{table}

\FloatBarrier

\section{Controlled synthetic probes: templates and scoring}
\label{app:controlled_details}

The main text reports results for three controlled probe families (pronoun
memory, entity tracking, structural closure). This section records the minimal
information needed to reproduce the probes.

\paragraph{Distances and sampling.}
We evaluate antecedent distances $d \in \{32, 64, 128, 256, 512, 1024\}$. For each
family and distance we generate a fixed number of examples with a shared filler
distribution across models.

\paragraph{Scoring.}
Pronoun memory and entity tracking are scored \emph{contrastively} at the target
position using (i) accuracy $\mathbf{1}[m>0]$ and (ii) a log-probability margin
$m = \log p(y^+ \mid c) - \log p(y^- \mid c)$, where $y^+$ is the correct target
and $y^-$ is a matched distractor under the same prefix $c$. Structural closure
is scored by the raw NLL on the closing token, $\ell_{\mathrm{close}} = -\log
p(y_{\mathrm{close}}\mid c)$.

\paragraph{Prompt templates.}
We use simple, highly regular templates so that the only manipulated factor is
antecedent distance.

\begin{itemize}[leftmargin=1.2em,itemsep=0.2em,topsep=0.2em]
  \item \textbf{Pronoun memory:} introduce two entities with different genders and
  roles; later, reference the role and force a pronoun choice (e.g., \tok{he}
  vs.\ \tok{she}).
  \item \textbf{Entity tracking:} bind two same-gender entities to distinct
  attributes; later, query the attribute and force a name choice.
  \item \textbf{Structural closure:} open a structure (e.g., an HTML tag or
  bracketed span), insert filler, and score the closing token.
\end{itemize}

\FloatBarrier

\section{Additional regression control diagnostics}
\label{app:control_coefficients}

The main text focuses on the dedicated repeated-$n$-gram analysis in
Figure~\ref{fig:bracket_copy_gap}. Here we report the remaining control-feature
diagnostics from the original coarse-tag regression. Figure~\ref{fig:appendix_control_coefficients}
separates two quantities: raw paired gaps for concrete repeated-token subsets,
and fitted coefficients from the frequency-adjusted regression in
Equation~\ref{eq:prose_reg}. The raw gaps ask whether the hybrid has lower NLL
on a subset in absolute terms; the coefficients ask how each feature shifts the
paired hybrid--transformer gap after controlling for domain, tag, word position,
sequence position, difficulty, previous-token distance, and target-token
frequency.

The sign convention is the same as in the main text. Positive raw gaps mean
lower hybrid NLL, while positive coefficients increase the adjusted hybrid
advantage. The repeated-token coefficients are negative or near zero, consistent
with the conclusion that visible-prefix reuse reduces the hybrid advantage once
comparable tokens are matched by the regression. The remaining controls are
included as adjustment variables rather than central findings.

\begin{figure}[t]
\centering
\includegraphics[width=0.86\linewidth]{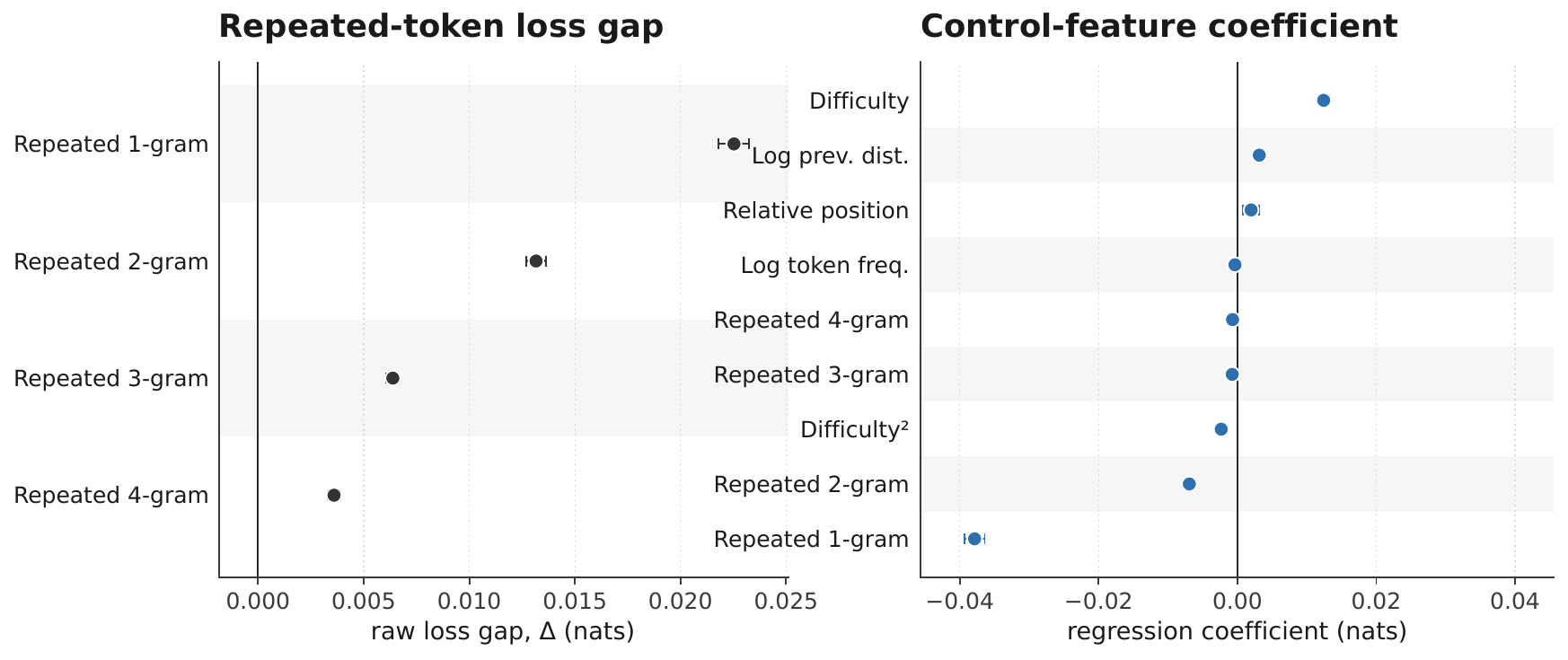}
\caption{\textbf{Full control-feature coefficient diagnostic.}
Left: raw paired loss gaps for repeated 1--4-gram events. Right: frequency-adjusted
coefficients from the coarse-tag regression, including difficulty, position,
previous-token distance, target-token frequency, and repeated-token controls.
Raw gaps show absolute hybrid advantage on repeated-token subsets, whereas
coefficients show adjusted shifts in the paired hybrid--transformer gap after
the full regression specification.}
\label{fig:appendix_control_coefficients}
\end{figure}

\FloatBarrier

\section{\upd{Regression fit and uncertainty}}
\label{app:regression_uncertainty}
\upd{
Equation~\ref{eq:prose_reg} was fit to 62,116,540 paired token observations in 7,954 packed-sequence clusters. The coarse-tag and aggregate word-class specifications have $R^2=0.002100$ and $R^2=0.002025$, respectively. Thus, the measured covariates explain about 0.2\% of individual-token variation in $\Delta_i$; the model is not intended as a predictive or exhaustive account. We interpret Figure~\ref{fig:tag_raw_vs_effect}'s adjusted quantities as conditional mean contrasts. All intervals below are cluster-robust 95\% confidence intervals, with inference clustered by packed sequence.
\par
\begin{table}[H]
\centering
\caption{Aggregate word-class conditional mean contrasts in Figure~\ref{fig:tag_raw_vs_effect}.}
\label{tab:fig2_clustered_ci}
\begin{tabular}{lcc}
\toprule
Class & Adjusted contrast (nats) & 95\% CI \\
\midrule
Content & $+0.001266$ & $[+0.000949,+0.001583]$ \\
Function & $-0.002210$ & $[-0.002579,-0.001841]$ \\
\bottomrule
\end{tabular}
\end{table}
Both contrasts are small but distinguishable from zero; they identify one systematic component of the architecture gap even though most token-level variance remains unexplained.
\par
\begin{table}[H]
\centering
\caption{Opening-minus-closing paired loss-gap contrasts from Figure~\ref{fig:bracket_copy_gap}A.}
\label{tab:fig4_domain_ci}
\begin{tabular}{lcc}
\toprule
Domain & Open--close contrast (nats) & 95\% CI \\
\midrule
PG-19 & $+0.0683$ & $[-0.0211,+0.1578]$ \\
CC-News & $+0.0774$ & $[+0.0383,+0.1164]$ \\
Wikipedia & $+0.0487$ & $[+0.0382,+0.0591]$ \\
arXiv & $+0.0279$ & $[+0.0203,+0.0355]$ \\
Python & $+0.0167$ & $[+0.0101,+0.0233]$ \\
HTML & $+0.0100$ & $[+0.0061,+0.0139]$ \\
\LaTeX{} & $+0.0640$ & $[+0.0569,+0.0710]$ \\
\bottomrule
\end{tabular}
\end{table}
The point estimate is positive in all seven domains, and the interval excludes zero in six. PG-19 is directionally positive but inconclusive because its interval includes zero.
}

\section{Limitations}
\label{app:limitations}

Our analysis has several limitations.
First, our coarse POS, bracket, and copy tags are noisy proxies for the
underlying computation a token requires. A noun token can be a
state-conditioned readout in one context and a near-copy in another; regression
controls reduce but do not eliminate this confounding.
Second, the tag-stratified and regression analyses are correlational. The
synthetic probes provide controlled manipulation, but each isolates a narrow
construction and does not directly measure the contribution of individual
recurrent layers.
Third, our interpretation in \S\ref{sec:discussion} is a description of where the
gap concentrates, not a mechanistic claim. We do not localize the hybrid
advantage to specific layers or heads, and dynamic-semantics formalisms are used
as a high-level lens rather than as a probed representation.
Fourth, all evaluations are in English prose, Python, HTML, and \LaTeX{}; we do
not test other natural languages or programming languages.

\section{Compute resources}
\label{app:compute}

\upd{The Olmo token-level analyses} use forward passes only (no training) on the released
Olmo~3~7B and Olmo Hybrid~7B checkpoints. Scoring the full evaluation suite
(packed length-8192 sequences across the four prose and three structured
domains) requires approximately 100 GPU-hours on a single 8xH100 node. The 1B
filtered-evaluation curves in \S\ref{sec:higher_signal_evals} re-score released
WSD-annealed checkpoints from \citet{merrill2026olmohybrid} and require
approximately 20 additional GPU-hours.

\section{Licenses for existing assets}
\label{app:licenses}

We use the following released artifacts.
\textbf{Models:} Olmo~3~7B and Olmo Hybrid~7B \citep{olmo2025olmo3,
merrill2026olmohybrid} are released under the Apache 2.0 license; the 1B
development checkpoints are released under the same terms. \upd{The three NVIDIA checkpoints in Appendix~\ref{app:nvidia_replication} are also released under Apache 2.0.}
\textbf{Datasets:} PG-19 (Apache 2.0); CC-News (CC-BY); Wikipedia
(CC BY-SA 3.0/4.0 / GFDL); ArXiv articles (per individual paper licenses;
we use only abstracts/text exposed via the public dump).
\textbf{Tools:} NLTK (Apache 2.0) for the Brown POS tagset; the Python
\texttt{tokenize} module (PSF license). We use these assets in accordance with
their licenses for non-commercial research.

\section{\upd{Filtered metrics at 7B}}
\label{app:filtered_eval_7b}
\upd{
To test whether the higher-signal evaluation in \S\ref{sec:higher_signal_evals} persists at the paper's main model scale, we evaluate 30 step-matched Olmo~3~7B and Olmo Hybrid~7B checkpoint pairs on a fixed six-domain probe, using the same filters as in \S\ref{sec:higher_signal_evals}. We report paired Transformer-minus-Hybrid NLL gaps, so positive values indicate lower Hybrid loss. Confidence intervals are 95\% packed-sequence bootstrap intervals.

\par
This analysis is intended as a preliminary diagnostic rather than an independent confirmatory test. Because the \textsc{Top-10} token families were selected using the final 7B checkpoint, the $2.04\times$ amplification at that checkpoint is partly selection-coupled. However, only the final checkpoint informed the filter: the filtered gap's consistency across the preceding 29 checkpoints is not guaranteed by construction and provides a trajectory-level robustness check, although checkpoints are correlated and this is not a held-out replication.

\begin{table}[H]
\centering
\caption{Filtered-loss results at the final 7B checkpoint. Gaps are paired Transformer-minus-Hybrid NLL differences in nats.}
\label{tab:filtered_eval_7b}
\begin{tabular}{lcc}
\toprule
Filter & Final gap [95\% CI] & Relative to all-token gap \\
\midrule
\textsc{All tokens} & $+0.0317\;[+0.0289,+0.0349]$ & $1.00\times$ \\
\textsc{Top-10}$\,\cap\,$\textsc{No-Copy} & $+0.0646\;[+0.0589,+0.0712]$ & $2.04\times$ \\
\textsc{Copy-5} & $+0.0025\;[+0.0016,+0.0033]$ & $0.08\times$ \\
\bottomrule
\end{tabular}
\end{table}

\paragraph{Checkpoint summary.}
The \textsc{Top-10}$\,\cap\,$\textsc{No-Copy} gap is larger than the all-token gap in all six domains and at all 30 step-matched checkpoints. At every checkpoint, the pointwise 95\% packed-sequence bootstrap interval for the filtered-minus-all difference excludes zero. In contrast, \textsc{Copy-5} retains only $0.08\times$ the all-token gap. Thus, while the final-checkpoint magnitude should be interpreted descriptively, the consistent training trajectory suggests that the higher-signal ordering persists at 7B scale.
}

\section{\upd{Replication on NVIDIA's Transformer--Mamba-2 family}}
\label{app:nvidia_replication}

\upd{
\paragraph{Models and setup.}
We repeat the three natural-token analyses without additional training using NVIDIA's pure-attention Transformer (\path{nvidia/gpt3-8b-multi-3.5t-base}), attention--Mamba-2 Hybrid (\path{nvidia/mamba2-hybrid-8b-3t-4k}), and pure Mamba-2 state-space model (\path{nvidia/mamba2-8b-3t-4k}) \citep{waleffe2024empirical}.
These approximately 8B checkpoints share a provider, SentencePiece tokenizer, 4,096-token context, and reported training data, hyperparameters, and 3.5T-token budget.
For the Transformer--Hybrid comparison, we retain the main text's sign convention, $\Delta_i=\ell_i^{\mathrm{Tr}}-\ell_i^{\mathrm{Hyb}}$, so positive values favor the Hybrid.

\paragraph{Pure recurrent baseline.}
The pure Mamba-2 model has approximately 8.24B parameters and 56 Mamba-2 blocks, with no attention layers, providing a pure recurrent state-space baseline.
It is distinct from the 1B Gated DeltaNet model labeled ``Pure RNN'' in \S\ref{sec:higher_signal_evals}.
Mamba-2 is a structured state-space model rather than a conventional RNN, so its results should not be generalized to conventional RNNs.

\paragraph{Transformer--Hybrid results.}
All three qualitative natural-token patterns replicate (Table~\ref{tab:nvidia_replication}).

\begin{table}[H]
\centering
\upd{
\caption{\upd{Replication on NVIDIA's Transformer--Mamba-2 Hybrid pair. Loss differences are in nats. The content--function and opening--closing contrasts compare the paired Transformer-minus-Hybrid gaps between token categories.}}
\label{tab:nvidia_replication}
\small
\begin{tabularx}{\linewidth}{@{}p{0.24\linewidth}X@{}}
\toprule
\textbf{Pattern} & \textbf{Result} \\
\midrule
Content vs.\ function & The pooled content-minus-function contrast is $+0.0102$ [$+0.0097$, $+0.0106$]. After adjustment for difficulty, frequency, subword position, and copying, it remains positive: $+0.0015$ [$+0.0010$, $+0.0020$]. \\
\addlinespace
Opening vs.\ closing & Opening-minus-closing is positive in all seven domains; four pointwise 95\% confidence intervals exclude zero. \\
\addlinespace
Repeated $n$-grams & The Hybrid advantage decreases from $+0.0251$ on repeated 1-grams to $+0.0004$ on repeated 16-grams, a paired reduction of $+0.0247$ [$+0.0243$, $+0.0251$]. \\
\bottomrule
\end{tabularx}
}
\end{table}

\paragraph{Comparison with pure Mamba-2.}
Defining the paired gap as Transformer NLL minus pure-Mamba-2 NLL, the gap also contracts between repeated 1-grams and repeated 16-grams, by $+0.0176$ [$+0.0170$, $+0.0182$] nats.
Thus, repetition reduces the Transformer--pure-Mamba-2 gap as well as the Transformer--Hybrid gap.
However, the full three-model ordering is not monotone in recurrence: these results do not support a general claim that more recurrence improves performance.

\paragraph{Architectural differences and interpretation.}
Despite their matched reported pretraining conditions, the NVIDIA Transformer and Hybrid differ in several architectural choices beyond the sequence mixer \citep{waleffe2024empirical}.
The Transformer uses LayerNorm, RoPE, multi-head attention, and SwiGLU, whereas the Hybrid uses RMSNorm, no explicit positional embeddings, grouped-query attention, and GELU.
Consequently, these released-checkpoint comparisons cannot isolate the contribution of recurrence from other architectural and parameterization differences, or rule out every training-recipe difference.
The replication supports the robustness of the observed token-level patterns across model families, but does not establish that recurrence alone causes them.\section{\texorpdfstring{\upd{Additional Related Work}}{Additional Related Work}}\label{app:copying_related_work}\upd{\textbf{Copying and retrieval.} Transformer induction heads provide a mechanism for in-context copying \citep{olsson2022incontext}. \emph{Repeat After Me} directly compares Transformers and state-space models on copying \citep{jelassi2024repeat}, while RetriVit evaluates $n$-gram and position retrieval across Transformers, SSMs, and hybrids \citep{pantazopoulos2026retrievit}. Our paired analysis instead asks where the Transformer--Hybrid loss gap concentrates.}
}

\upd{
\section{Proofs of Theoretical Claims} \label{app:theory-proofs}

\zeroError*
\begin{proof}
    For any $X_t$ of type $\tau$, we have by construction that $X_t$ is perfectly predictable from $\phi(X_{<t})$, and thus
    \begin{equation*}
        p^*(X_t \mid X_{<t}) = p_\phi(X_t \mid \phi(X_{<t})) .
    \end{equation*}
    It follows that the KL divergence is 0.
\end{proof}

\klUpperBound*
\begin{proof}
    Because $p_{\phi}$ is the best predictor under the feature restriction $\phi$ and the class restriction $X_t\in\mathcal{V}_{\tau}$, \upd{its error is at most the error of a predictor that predicts tokens from $\mathcal V_\tau$ uniformly.
    Letting $U_{\tau}$ be the uniform distribution on $\mathcal{V}_{\tau}$, we have:}
    \begin{align*}
        D_{\mathrm{KL}}^\tau (p^\star\|p_{\phi})
        &\le
        D_{\mathrm{KL}}^\tau(p^\star\|U_{\tau})
        \le \log |\mathcal{V}_{\tau}|. \qedhere
    \end{align*}
\end{proof}

}

\newpage
\section*{NeurIPS Paper Checklist}

\begin{enumerate}

\item {\bf Claims}
    \item[] Question: Do the main claims made in the abstract and introduction accurately reflect the paper's contributions and scope?
    \item[] Answer: \answerYes{} 
    \item[] Justification: The abstract and introduction state our two guiding questions and preview the three empirical findings (content-word advantage, opening-vs-closing asymmetry, vanishing advantage on repeated
n-grams) that are demonstrated in Section~\ref{sec:empirical_results}, and the discourse-state interpretation and filtered-evaluation application in Sections~\ref{sec:discussion}--\ref{sec:higher_signal_evals}.

\item {\bf Limitations}
    \item[] Question: Does the paper discuss the limitations of the work performed by the authors?
    \item[] Answer: \answerYes{} 
    \item[] Justification: See \Cref{app:limitations} in appendix. 

\item {\bf Theory assumptions and proofs}
    \item[] Question: For each theoretical result, does the paper provide the full set of assumptions and a complete (and correct) proof?
    \item[] Answer: \answerYes{}
    \item[] Justification: Section~\ref{sec:discussion} gives two idealized propositions, with proofs in Appendix~\ref{app:theory-proofs}. The proofs use the best class-conditional predictor under the feature restriction; the first proposition additionally assumes a deterministic target given those features.

    \item {\bf Experimental result reproducibility}
    \item[] Question: Does the paper fully disclose all the information needed to reproduce the main experimental results of the paper to the extent that it affects the main claims and/or conclusions of the paper (regardless of whether the code and data are provided or not)?
    \item[] Answer: \answerYes{} 
    \item[] Justification: We use the publicly released open-weight checkpoints of Olmo~3~7B and Olmo Hybrid~7B at the matched final checkpoint, and the publicly released 1B development runs from \citet{merrill2026olmohybrid}. Section~\ref{sec:empirical_method} and Appendix~\ref{app:method_details} describe the packing length, pairing protocol, tag taxonomies and alignment rules, regression specification, and synthetic probe construction. Public datasets (PG-19, CC-News, Wikipedia, ArXiv, plus Python/HTML/\LaTeX{} sources) are cited in Appendix~\ref{app:domains_details}.

\item {\bf Open access to data and code}
    \item[] Question: Does the paper provide open access to the data and code, with sufficient instructions to faithfully reproduce the main experimental results, as described in supplemental material?
    \item[] Answer: \answerNo{} 
    \item[] Justification: We plan to release the analysis code at the camera-ready stage.

\item {\bf Experimental setting/details}
    \item[] Question: Does the paper specify all the training and test details (e.g., data splits, hyperparameters, how they were chosen, type of optimizer) necessary to understand the results?
    \item[] Answer: \answerYes{} 
    \item[] Justification: Section~\ref{sec:empirical_method} gives the paired-NLL protocol, sequence packing length (T=8192), tag taxonomies, regression specification (Equation~\ref{eq:prose_reg}), and synthetic probe scoring. Appendix~\ref{app:method_details} provides the full domain list, alignment rules for multi-tag attribution, regression covariates, and probe construction. The 1B pretraining runs in \S\ref{sec:higher_signal_evals} use the released training configurations from \citet{merrill2026olmohybrid}.

\item {\bf Experiment statistical significance}
    \item[] Question: Does the paper report error bars suitably and correctly defined or other appropriate information about the statistical significance of the experiments?
    \item[] Answer: \answerYes{} 
    \item[] Justification: CIs are included in tag-bar and regression coefficients plots.

\item {\bf Experiments compute resources}
    \item[] Question: For each experiment, does the paper provide sufficient information on the computer resources (type of compute workers, memory, time of execution) needed to reproduce the experiments?
    \item[] Answer: \answerYes{} 
    \item[] Justification: See Appendix \Cref{app:compute}.
    
\item {\bf Code of ethics}
    \item[] Question: Does the research conducted in the paper conform, in every respect, with the NeurIPS Code of Ethics \url{https://neurips.cc/public/EthicsGuidelines}?
    \item[] Answer: \answerYes{}
    \item[] Justification: The work uses only publicly released model checkpoints and standard public text/code corpora, and does not involve human subjects, sensitive data, or release of new models. We have reviewed the NeurIPS Code of Ethics and find no concerns.

\item {\bf Broader impacts}
    \item[] Question: Does the paper discuss both potential positive societal impacts and negative societal impacts of the work performed?
    \item[] Answer: \answerYes{} 
    \item[] Justification: Its main societal effect is indirect: better diagnostics may accelerate the development of more capable language models. See \Cref{sec:discussion}.
    
\item {\bf Safeguards}
    \item[] Question: Does the paper describe safeguards that have been put in place for responsible release of data or models that have a high risk for misuse (e.g., pre-trained language models, image generators, or scraped datasets)?
    \item[] Answer: \answerNA{} 
    \item[] Justification: We do not release any new pretrained models, datasets, or other assets that pose a high risk for misuse. The analysis uses already-released open-weight checkpoints.

\item {\bf Licenses for existing assets}
    \item[] Question: Are the creators or original owners of assets (e.g., code, data, models), used in the paper, properly credited and are the license and terms of use explicitly mentioned and properly respected?
    \item[] Answer: \answerYes{} 
    \item[] Justification: See Appendix \Cref{app:licenses}.

\item {\bf New assets}
    \item[] Question: Are new assets introduced in the paper well documented and is the documentation provided alongside the assets?
    \item[] Answer: \answerNA{} 
    \item[] Justification: We do not release new datasets or model weights. 

\item {\bf Crowdsourcing and research with human subjects}
    \item[] Question: For crowdsourcing experiments and research with human subjects, does the paper include the full text of instructions given to participants and screenshots, if applicable, as well as details about compensation (if any)? 
    \item[] Answer: \answerNA{} 
    \item[] Justification: The paper does not involve crowdsourcing or human subjects. All evaluations are computed automatically over public corpora.

\item {\bf Institutional review board (IRB) approvals or equivalent for research with human subjects}
    \item[] Question: Does the paper describe potential risks incurred by study participants, whether such risks were disclosed to the subjects, and whether Institutional Review Board (IRB) approvals (or an equivalent approval/review based on the requirements of your country or institution) were obtained?
    \item[] Answer: \answerNA{} 
    \item[] Justification: The paper does not involve human subjects research.

\item {\bf Declaration of LLM usage}
    \item[] Question: Does the paper describe the usage of LLMs if it is an important, original, or non-standard component of the core methods in this research? Note that if the LLM is used only for writing, editing, or formatting purposes and does \emph{not} impact the core methodology, scientific rigor, or originality of the research, declaration is not required.
    \item[] Answer: \answerNA{} 
    \item[] Justification: LLMs (Olmo~3, Olmo Hybrid) are the object of study, not a methodological component. Any LLM use for writing or editing assistance was limited to surface polishing and did not affect the methodology, analyses, or conclusions.

\end{enumerate}

\end{document}